\documentclass[a4paper,12pt]{article}

\usepackage{color}
\usepackage{authblk}
\usepackage{amsmath,amsfonts,amssymb,graphicx}
\usepackage{natbib}

\usepackage[boxruled]{algorithm2e}
\SetAlFnt{\normalsize}
\SetAlCapFnt{\normalsize}
\SetAlCapNameFnt{\normalsize}

\makeatletter
\renewcommand{\algocf@caption@boxruled}{%
	\hrule
	\hbox to \hsize{%
		\vrule\hskip-0.4pt
		\vbox{   
			\vskip\interspacetitleboxruled%
			\vskip.07cm
			\hspace{.2cm}\unhbox\algocf@capbox\hfill
			\vskip.05cm
			\vskip\interspacetitleboxruled
		}%
		\hskip-0.4pt\vrule%
	}\nointerlineskip%
}%
\makeatother

\usepackage{enumerate}

\def\beginrefs{\begin{list}%
		{[\arabic{equation}]}{\usecounter{equation}
			\setlength{\leftmargin}{2.0truecm}\setlength{\labelsep}{0.4truecm}%
			\setlength{\labelwidth}{1.6truecm}}}
	\def\endrefs{\end{list}}

\newtheorem{theorem}{Theorem}
\newtheorem{lemma}[theorem]{Lemma}
\newtheorem{proposition}[theorem]{Proposition}

\newtheorem{definition}[theorem]{Definition}
\newtheorem{remark}[theorem]{Remark}
\newtheorem{assumption}[theorem]{Assumption}

\newenvironment{proof}{{\bf Proof:}}{\hfill\rule{2mm}{2mm}}

\newcommand\E{\mathbb{E}}
\newcommand\J{\mathcal{J}}

\newcommand\Z{\mathcal{Z}}
\newcommand\A{\mathcal{A}}

\newcommand\R{\mathbb{R}}

\renewcommand\S{\mathcal{S}}

\newcommand\F{\mathcal{F}}
\newcommand\G{\mathcal{G}}

\newcommand\1{\mathbf{1}}

\newcommand{\K}{\mathcal{K}}

\newcommand{\argmin}{\mathop{\mathrm{argmin}}}  
\newcommand{\argmax}{\mathop{\mathrm{argmax}}}

\newcommand{\norm}[1]{\left\lVert #1 \right\rVert}

\usepackage{listings}
\usepackage{xcolor}

\definecolor{codegreen}{rgb}{0,0.6,0}
\definecolor{codegray}{rgb}{0.5,0.5,0.5}
\definecolor{codepurple}{rgb}{0.58,0,0.82}
\definecolor{backcolour}{rgb}{0.95,0.95,0.92}

\lstdefinestyle{mystyle}{
	backgroundcolor=\color{backcolour},   
	commentstyle=\color{codegreen},
	keywordstyle=\color{magenta},
	numberstyle=\tiny\color{codegray},
	stringstyle=\color{codepurple},
	basicstyle=\linespread{1}\ttfamily\footnotesize,
	breakatwhitespace=false,         
	breaklines=true,                 
	captionpos=b,                    
	keepspaces=true,                 
	numbers=left,                    
	numbersep=5pt,                  
	showspaces=false,                
	showstringspaces=false,
	showtabs=false,                  
	tabsize=2
}

\title{Dimension-Free Rates for Natural Policy Gradient in Multi-Agent Reinforcement Learning}

\author[*]{Carlo Alfano}
\author[*]{Patrick Rebeschini}
\affil[*]{Department of Statistics, University of Oxford}
\date{\vspace{-5ex}}
\begin{document}

\maketitle

\begin{abstract}
  Cooperative multi-agent reinforcement learning is a decentralized paradigm in sequential decision making where agents distributed over a network iteratively collaborate with neighbors to maximize global (network-wide) notions of rewards. Exact computations typically involve a complexity that scales exponentially with the number of agents. To address this curse of dimensionality, we design a scalable algorithm based on the Natural Policy Gradient framework that uses local information and only requires agents to communicate with neighbors within a certain range. Under standard assumptions on the spatial decay of correlations for the transition dynamics of the underlying Markov process and the localized learning policy, we show that our algorithm converges to the globally optimal policy with a dimension-free statistical and computational complexity, incurring a localization error that does not depend on the number of agents and converges to zero exponentially fast as a function of the range of communication.
\end{abstract}

\section{Introduction}
\label{intro}

Sequential decision-making is a prominent setting in modern statistical theories and applications, where agents sequentially interact with an environment—observing its state, taking actions, and receiving rewards—to maximize notions of reward. Reinforcement learning is the setting where agents do not have complete knowledge of the environment dynamics, and it has received increased attention due to its recent successes on a variety of domains, e.g.\ games \citep{RN19, RN20} and autonomous driving \citep{RN18}.

Modern applications typically involve high-dimensional state and action spaces, and classical algorithms often lead to a computational complexity that scales exponentially with the number of degrees of freedom in the model. Understanding which structures can be used to design approximate methods that can overcome this curse of dimensionality while retaining near-optimal statistical guarantees is a question of paramount importance.

A class of algorithms that have proven successful to face high-dimensional models is that of Natural Policy Gradient (NPG) methods \citep{RN61,RN39, RN63, RN64}. It has recently been shown \citep{RN28} that NPG converges to an optimal policy with an iteration complexity that scales only logarithmically with the cardinality of the action space and with no explicit dependence on the cardinality of the state space.

Despite the favorable iteration complexity of NPG, NPG still faces the curse of dimensionality in applications where the computational cost per iteration scales exponentially with the number of degrees of freedom. This is the case in the setting of multi-agent reinforcement learning (MARL), for instance, where agents distributed over a network iteratively interact with each other to maximize global notions of reward. In this setting, the computational complexity of NPG---when applied to the entire network of agents---scales exponentially with the dimension of the model, which corresponds to the number of agents (see Section \ref{sec:curse}).

Along with the curse of dimensionality, NPG also faces scalability and implementability issues within the MARL framework. Applying NPG to the entire network of agents requires global communication, i.e.\ it requires each agent to be able to communicate with every other agent in the network, at every time step. This requirement is unrealistic in many multi-agent applications of interest, where the network topology is typically sparse, often grid-like, and where agents are only allowed to perform computation and communication in a decentralized manner, interacting only with neighboring agents within a certain range. These computational and communicational constraints arise, for instance, in the case of sensor networks, e.g.\ \citep{RN71,RN78}, and in the case of intelligent transportation systems, e.g.\ \citep{RN72}.

Over the past decades, various approaches have been proposed to address the curse of dimensionality in high-dimensional reinforcement learning models and, before that, in high-dimensional dynamical programming models, where exact knowledge of the probabilistic structure describing the environment is assumed. A popular approach involves designing algorithms that can exploit notions of \emph{locality}, which encodes the assumption that, in some regimes, information can dissipate when it propagates through the network so that global computation and communication are not required to meet a prescribed level of error accuracy. Exploiting locality prompted the use of ad-hoc approximate factorization and truncation techniques, such as expressing the value function as a linear combination of basis functions that only depend on a small subset of local variables \citep{RN14,RN13,  RN16, RN43, RN42}. These ideas have been applied to the MARL setting \citep{RN12, RN15,RN21,RN17,RN35,RN34, RN33} and have proven successful in experiments, but lack theoretical guarantees or non-asymptotic analysis. A recent line of work has formally considered spatial decay of correlation assumptions for nearest-neighbors dynamics and designed decentralized algorithms based on policy gradient and actor-critic methods \citep{RN11, RN26, RN30, RN31}, establishing non-asymptotic convergence guarantees towards a \emph{stationary} point, but not towards an \emph{optimal} policy.\footnote{Remark \ref{rem:harv} gives a complete comparison of our results against previous findings that exploit the same type of decay of correlation assumption.} An application of the same principles to the setting of \emph{mean-field} MARL \citep{RN88} can be found in \cite{RN87}, where the authors show that a neural network based version of the actor-critic algorithm can achieve global convergence. In this setting, however, agents are considered to be indistinguishable and the transition scheme of an agent is only affected by the
mean effect from its neighbors.

In this paper, we design a decentralized algorithm for the MARL setting based on the NPG framework that only uses local computations and communication for neighbors of agents within a certain range. We show that our algorithm can provably exploit spatial decay of correlation properties to overcome the curse of dimensionality, establishing non-asymptotic convergence guarantees to a globally \emph{optimal} policy. In particular, we consider a general formulation of the decay of correlation assumption from statistical mechanics and probability theory \citep{RN50, RN51,RN75}, whereby agents have an influence on each other that decays exponentially with their distance on the network. This type of assumption has been previously considered in the learning literature, e.g. in \cite{RN86, RN82, RN85, RN84, RN81}, and also in the MARL setting, c.f.\ discussion in the previous paragraph. Under this assumption, we derive convergence bounds that are the same as those established for (centralized) NPG in \citep{RN28}, worsened only by a localization error that decreases exponentially with the radius of the communication range. A key feature of our bounds is that they are \emph{dimension-free}, as they do not depend on the number of agents, and depend only logarithmically on the cardinality of the action space of \emph{individual} agents and do not explicitly depend on the state space of individual agents. The localization radius controls the trade-off between statistical accuracy and computational complexity, as the overall computational cost of our algorithm scales only with respect to the number of agents within the local communication radius, and not with the total number of agents in the network.

Our contribution fits into the more general literature that has shown how spatial decay of correlations can be used to establish dimension-free results and are of interest in a variety of settings, such as \citep{RN52, RN53}, mixing times in spin systems \citep{RN47,RN48}, particle filtering \citep{RN49}, epidemics \citep{RN56}, social networks \citep{RN57}, communication networks \citep{RN58}, queuing networks \citep{RN59}, and smart transportation \citep{RN60}.

The paper organization is as follows. In Section 2 we describe the MARL framework and we discuss the model assumptions we work with. In Section 3 we describe NPG as presented in \cite{RN28} and discuss its limitations when applied to MARL. In Section 4 we design a decentralized version of MARL and state our main results. The Appendix contains all the proofs of our statements and elaborates on the model assumptions.

\section{Setting}
\label{setting}
Let $\G = (\K,\mathcal{E})$ be an undirected graph describing a network of $|\K|=K$ agents. On this graph, the distance $d(k,k')$ between two agents $k,k'\in\K$ is defined as the length of the shortest path between the two vertices. Let $N^r_k = \{k' \in \K : d(k,k') \leq r\}$ denote the neighborhood of radius $r$ of agent $k$, with $N_k=N^1_k$ and $N^r_{-k}= \K \smallsetminus N^r_k$. Let $S_k$ and $A_k$ be the state and action spaces associated with agent $k$. We consider a Markov Decision Process (MDP) $(\S,\A,P,r,\gamma,\mu)$: $\S = \S_1 \times\dots\times\S_K$ and $\A = \A_1 \times\dots\times\A_K$ are, respectively, the global state space and the global action space; $\forall s,s' \in \S, a \in \A$, $P_k(s'_k | s,a)$ is the local transition probability, that is the probability that agent $k$ transitions to state $s'_k$ when the global state and action is $(s,a)$, and $P(s'|s,a) = \prod_{k \in \K}P_k(s'_k | s,a)$ is the global transition probability; $r(s,a)= \frac{1}{K}\sum_{k\in\K}r_k(s_k,a_k)$ is the global (network-wide) reward function that we wish to maximize, where $r_k: \S_k\times\A_k\rightarrow[0,1]$ is the reward for agent $k$; $\gamma$ is the discount factor and $\mu$ is the starting state distribution. At time $t$ denote the current state and action by $s(t)$ and $a(t)$.

To each agent $k$ is assigned a local differentiable policy parameterized by $\theta_k \in \Theta_k$,
\[\pi_{\theta_k}(a_k|s)= \frac{e^{f_{\theta_k}(s, a_k)}}{\sum_{a' \in \A_k}e^{f_{\theta_k}(s,a')}},\]
which depends on the current global state $s$. Given the current global state $s$, each agent acts independently of the others. Denote $\theta = (\theta_1,\dots,\theta_K)$ and $\Theta= \Theta_1\times\dots\times\Theta_K$, then $\pi_\theta(a|s) = \prod_{k \in \K}\pi_{\theta_k}(a_k|s)$.

For a policy $\pi$ and for each agent $k$, let $V_k^\pi:\S\rightarrow\R$ be the respective value function, which is defined as the expected discounted cumulative reward with starting state $s(0)=s$, namely,
\[V^\pi_k(s)= \mathbb{E}\left[\sum_{t=0}^{\infty}\gamma^t r_k(s_k(t),a_k(t)) \bigg| \pi, s(0) = s\right],\]
where $a(t)\sim\pi(\cdot|s(t))$ and $s(t+1)\sim P(\cdot|s(t),a(t))$. Let $V^\pi$ be the global value function, defined as $V^\pi(s)=\frac{1}{K}\sum_{k\in\K}V^\pi_k(s)$, and $V^\pi(\mu)$ be the expected global value function when the initial state distribution is $\mu$, i.e.\  $V^\pi(\mu)=\E_{s\sim\mu}V^\pi(s)$. Our objective is to find an optimal policy $\pi^\star\in \argmax_\pi\E_{s\sim\mu}V^\pi(s)$.

For a policy $\pi$ and for each agent $k$, let $Q_k^\pi:\S\times\A\rightarrow\R$ be the respective Q-function, which is defined as the expected discounted cumulative reward with starting state $s(0)=s$ and starting action $a(0)=a$, namely,
\[Q^\pi_k(s, a) =\E\left[\sum_{t=0}^\infty \gamma^t r_k(s_k(t), a_k(t))\bigg|\pi, s(0)=s, a(0)=a\right],\]
where $a(t)\sim\pi(\cdot|s(t))$ and $s(t+1)\sim P(\cdot|s(t),a(t))$. Let $Q^\pi$ be the global value function, defined as $Q^\pi(s,a)= \frac{1}{K} \sum_{k\in\K} Q^\pi_k(s,a)$.

Let $A_k^\pi:\S\times\A\rightarrow\R$ be the advantage function for policy $\pi$ and agent $k$, representing the advantage of taking the action $a$ at step $0$ and then following policy $\pi$, with respect to following policy $\pi$ from the start, and defined as \[A_k^\pi(s,a) = Q_k^\pi(s,a)-V_k^\pi(s).\] Let $A^\pi(s,a)= \frac{1}{K} \sum_{k\in\K} A^\pi_k(s,a)$ be the global advantage function.

We define the discounted state visitation distribution \citep{RN37}, 
\[	d_{\rho}^\pi(s)= (1-\gamma)\E_{s(0) \sim \rho}\sum_{t=0}^{\infty}\gamma^tP(s(t)=s|\pi, s(0)),\]
and the  discounted state-action visitation distribution,
\[d^{\pi}_\nu(s,a) = (1-\gamma)\E_{s(0),a(0) \sim\nu}\sum_{t=0}^{\infty}\gamma^tP(s(t)=s, a(t)=a|\pi,s(0),a(0)),\]
where the trajectory $(s(t),a(t))_{t\geq0}$ is generated by the MDP following policy $\pi$. Lastly, a function $f:\Theta\rightarrow\R$ is said to be a $\delta$-smooth function of $\theta$ if, $\forall\theta,\theta'\in\Theta$,
\[\norm{\nabla f(\theta)-\nabla f(\theta')}_2\leq \delta\norm{\theta-\theta'}_2.\]

\subsection{Model Assumptions}

We assume that a version of the Dobrushin condition \citep{RN75} holds for the transition dynamics of the network of agents. Let $TV(\mu,\nu)=\sup_{A\in\F}\left|\mu(A)-\nu(A)\right|$ be the total variation distance between the probability distributions $\mu$ and $\nu$ defined on the $\sigma$-algebra $\F$.
\begin{assumption}(Spatial Decay of Correlation for the Dynamics)
	\label{ass:dobr}
	Let $C\in\R^{K\times K}$ be defined as follows:
	\[	C_{ij}=\sup_{s_j,s_{-j},a_j,a_{-j}, s'_j,a'_j} TV(P_i(\cdot|s_j,s_{-j},a_j,a_{-j}),P_i(\cdot|s'_j,s_{-j},a'_j,a_{-j})).\]
	Assume that there exists $\beta\geq0$ such that
	\[\max_{k\in\K}\sum_{j \in \K}e^{\beta  d(k,j)}C_{kj}\leq\rho,\]
	with $ \rho<1/\gamma$, where $\gamma$ is the discount factor of the MDP.
\end{assumption}
The element $(i,j)$ of the matrix $C$ represents the influence that a perturbation of the state and action of agent $j$ has on the transition probability of agent $i$. Assumption \ref{ass:dobr} encodes the fact that the transition dynamics of each agent is exponentially less sensible to perturbations of the state and action of further away agents. The requirement $\rho<1/\gamma$ comes as we need the spatial decay of correlation for the dynamics to be strong enough to induce a spatial decay for the Q-function (see Appendix \ref{app:exp_decay}). A small value of the discount factor $\gamma$ eases this requirement since it reduces the effect of perturbations through time. When $\beta=0$ and $\gamma=1$, Assumption \ref{ass:dobr} recovers the assumption in \cite{RN26} as a particular case.

Differently from \cite{RN26}, we require the Dobrushin condition to hold for the policy as well. This is due to the fact that Assumption \ref{ass:dobr} is sufficient to prove the decay of correlation for the Q-function, on which the algorithm of \cite{RN26} is based, but it is not sufficient to prove the decay of correlation for the value function, which instead needs an additional assumption on the policy. Since the NPG framework on which we build upon is based on both the Q-function and the value function, we make the following additional assumption.
\begin{assumption}(Spatial Decay of Correlation for the Policy)
	\label{ass:pol}
	Assume that there exist $\xi,\beta\geq0$ such that, $\forall \theta \in\Theta$,
	\[\sup_{s_{N^r_{k}},s_{N^r_{-k}},s'_{N^r_{-k}}}TV(\pi_{\theta_k}(\cdot|s_{N^r_{k}},s_{N^r_{-k}}),\pi_{\theta_k}(\cdot|s_{N^r_{k}}, s'_{N^r_{-k}}))\leq \xi e^{-\beta r}.\]
\end{assumption}

\begin{assumption}(Local Policy)
	\label{ass:pol2}
	Assume that, for any neighborhood radius $r$, the parameters $\theta_k$ can be partitioned in $(\theta_k)_{N^r_k}$ and $(\theta_k)_{N^r_{-k}}$ so that, if $(	\theta_k)_{N^r_{-k}}=0$, then \[\pi_{\theta_k}( a_k|s)= \pi_{\theta_k}(a_k|s_{N^r_{k}}),\]
	\[\nabla_{(\theta_k)_{N^r_{k}}}\log \pi_{\theta_k}( a_k|s)=\nabla_{(\theta_k)_{N^r_{k}}}\log \pi_{\theta_k}(a_k|s_{N^r_{k}}).\]
\end{assumption}
Assumptions \ref{ass:pol} and \ref{ass:pol2} impose a design constraint for the policy class $\{\pi_{\theta}\,|\,\theta\in\Theta\}$ rather than being assumptions on the nature of the environment, as the case for Assumption \ref{ass:dobr}. Assumption \ref{ass:pol} encodes, for the policy, a type of decay of correlation property that is  weaker than Assumption \ref{ass:dobr}. Assumption \ref{ass:pol} allows us to consider a policy class that presents properties that are necessary for the optimal policy under Assumption \ref{ass:dobr}, as we show in Appendix \ref{app:ass}. Assumption \ref{ass:pol2} is made to address the communication constraints of the network and requires the possibility of computing the policy and its gradient without access to the information coming from distant agents by setting their associated parameters to 0. In practice, we do only need Assumption \ref{ass:pol2} to hold for the value of $r$ we want Theorem \ref{thm:main} to hold for. In Appendix \ref{app:pol_ex}, we describe a policy class that satisfies both Assumption 2 and 3 for any value of $r$.

\subsection{Exponential Decay}
To take advantage of the local structure of the network, \cite{RN30} define a property regarding the dependence of $Q^\pi_k(s, a)$ on the neighbors of $k$. 
\begin{definition}
	[\cite{RN30}]
	The $(c,\psi)$-exponential decay property for the Q-function holds if, for any agent $k\in\K$ and for any $(s,a)$, $(\widetilde{s},\widetilde{a}) \in \S\times\A$ such that $s_{N^r_k} = \widetilde{s}_{N^r_k}, a_{N^r_k} = \widetilde{a}_{N^r_k}$, we have that \[\left|Q_k^\pi(s,a)-Q_k^\pi(\widetilde{s},\widetilde{a})\right|\leq c\psi^{r+1}.\]
\end{definition}
In our analysis, we need to define the exponential decay property for the value function as well. 
\begin{definition}
	The $(c',\phi)$-exponential decay property for the value function holds if, for any agent $k\in\K$ and for any $s$, $\widetilde{s}\in\S$ such that $s_{N^r_k} = \widetilde{s}_{N^r_k}$, we have that \[\left|V_k^\pi(s)-V_k^\pi(\widetilde{s})\right|\leq c'\phi^{r+1}.\]
\end{definition}
These two properties mean that the cumulative discounted rewards of agents have an exponential decaying dependence on the states and actions of distant agents. We show that both these properties hold in our setting.
\begin{proposition}
	\label{prop:decay}
	If Assumptions \ref{ass:dobr} and \ref{ass:pol} hold, then the exponential decay property holds for both the Q-function and the value function with parameters $(c,\psi)=\left(\frac{\gamma\rho e^{\beta}}{1-\gamma\rho},e^{-\beta}\right)$ and $(c',\phi)=\left(\frac{\gamma (\rho+\xi) e^{\beta }}{1-\gamma(\rho+\xi)},e^{-\beta}\right)$, respectively.
\end{proposition}

For clarity of exposition, in the rest of the paper we make the following assumption. 

\begin{assumption}
	\label{ass:exp}
	Assume that the exponential decay property holds for the Q-function with parameters $(c,\psi)$ and that it holds for the value function with parameters $(c',\phi)$.
\end{assumption}

\section{Natural Policy Gradient}
\label{NPG}
We consider NPG as presented in \cite{RN28}, which has iteration complexity that scales as $O(\sqrt{\log|\A|/T})$, where $T$ is the number of iterations. We now summarize the algorithm and the results in \cite{RN28} and show what problems arise in the multi-agent setting that we consider.

Let $\pi_{\theta}$ be a differentiable policy and define the Fisher information matrix induced by $\pi_{\theta}$ as
\[
	F_\mu(\theta)=\E_{s \sim d_{\mu}^\pi}\E_{a \sim \pi_\theta(\cdot|s)}\left[\nabla_\theta\log\pi_\theta(a|s) (\nabla_\theta\log\pi_\theta(a|s)) ^\top\right].
\]
The NPG update, with step-size $\eta$, is defined as
\begin{equation}
	\label{eq:npg_update}
	\theta^{(t+1)}= \theta^{(t)}+\eta F_\mu(\theta^{(t)})^{-1}\nabla_\theta V^{\theta^{(t)}}(\mu),
\end{equation}
where $\theta^{(t)}$ is the set of parameters at iteration $t$, $\nabla_\theta V^{\theta}(\mu)$ is the gradient of the value function with respect to the policy parameters, and $F_\mu(\theta^{(t)})^{-1}$ is the Moore-Penrose pseudo-inverse of the Fisher information matrix. As discussed in \cite{RN28}, the update in (\ref{eq:npg_update}) is equivalent to solving the problem
\begin{equation}
	w_\star \in \argmin_w\E_{s \sim d_{\mu}^{\pi_\theta}, a\sim\pi_{\theta}(\cdot|s)}\left[\left(A^{\pi_{\theta}}(s,a)-w\cdot\nabla_{\theta}\log\pi_\theta(\cdot|s)\right)^2\right]
\end{equation}
and then performing the following update:
\begin{equation}
	\label{eq:npg_update2}
	\theta^{(t+1)}= \theta^{(t)}+\frac{\eta}{1-\gamma}w^\star.
\end{equation}
Define \[L(w,\theta,\nu):=\E_{s,a \sim \nu}\left[\left(A^{\pi_\theta}(s,a)-w\cdot \nabla_\theta\log\pi_\theta(a|s)\right)^2\right].\]
Assume that $\log\pi_\theta(a|s)$ is a $\delta$-smooth function of $\theta$ and that $\pi^{(0)}$ is the uniform distribution. Let $d^{(t)} = d^{\pi_t}_\nu(s,a)$ and $d^\star(s,a)= d^{\pi^\star}_\mu(s) \pi^\star(a|s)$. Let $\nu$ be a distribution of $s,a$ such that \[\sup_{w\in\R^d}\frac{w^\top\Sigma^{(t)}_{d^\star}w}{w^\top\Sigma^{(t)}_{\nu}w} \leq \kappa,\] where \[\Sigma_\nu^\theta=\E_{s,a \sim \nu}\left[\nabla_\theta\log\pi_\theta(a|s)(\nabla_\theta\log\pi_\theta(a|s))^\top\right]\] and $\Sigma^{(t)} = \Sigma^{\theta^{(t)}}$. Lastly, assume that
\begin{align*}
\E\left[L(w^{(t)}_\star,\theta^{(t)},d^\star)\right]&\leq\varepsilon_\text{bias},\\
\E\left[L(w^{(t)},\theta^{(t)},d^{(t)})- L(w^{(t)}_\star,\theta^{(t)},d^{(t)})|\theta^{(t)}\right]&\leq\varepsilon_\text{stat},
\end{align*}
where
\begin{equation*}
	w^{(t)}_\star\in\argmin_{\norm{w}_2\leq W}L(w,\theta^{(t)},d^{(t)})
\end{equation*}
and the expectations are taken w.r.t.\ the sequence $(w^{(t)})_{t=0,\dots,T-1}$. Then, running algorithm (\ref{eq:npg_update2}) for $T$ time steps with $\eta = \sqrt{2 \frac{\log|\A|}{(\delta W^2T)}}$ for a given parameter $W$, we have the following guarantee: [\cite{RN28}, Theorem 6.2]
\begin{equation}
	\label{eq:agar}
	\E\left[\min_{t\leq T}\left\{V^{\pi^\star}(\mu)-V^{(t)}(\mu)\right\}\right]\leq\frac{W}{1-\gamma}\sqrt{\frac{2\delta\log|\A|}{T}}+ \sqrt{\frac{\kappa\varepsilon_\text{stat} }{(1-\gamma)^3}}+\frac{\sqrt{\varepsilon_\text{bias}}}{1-\gamma}.
\end{equation}
As we highlighted in the introduction, the guarantee \eqref{eq:agar} is particularly suitable for high-dimensional settings, as there is no explicit dependence on $|\S|$ and the dependence on $|\A|$ is only logarithmic. As to implicit dependencies, \cite{RN28} state that it is reasonable to expect that $\kappa$ is not a quantity related to $|\S|$. On the other hand, $\varepsilon_\text{stat}$ and $\varepsilon_\text{bias}$ are constants related to a minimization problem that depends on both $\S$ and $\A$.
\subsection{Curse of Dimensionality and Scalability/Implementabily Issues in MARL}
\label{sec:curse}
When applied to the MARL setting, with $\S = \S_1 \times\dots\times\S_K$ and $\A = \A_1 \times\dots\times\A_K$, NPG would incur a curse of dimensionality or scalability and implementabily issues, depending on the approach used for the minimization problem in \eqref{eq:min}. The guarantees for the algorithm in \cite{RN28} would yield:
\begin{equation}
	\label{eq:gua}
	\E\left[\min_{t<T}\{V^{\pi^\star}(\rho)-V^{(t)}(\rho)\}\right]\leq \frac{WK}{1-\gamma}\sqrt{\frac{2\delta\log\max_{k\in\K}|\A_k|}{T}}+\sqrt{\frac{\kappa\varepsilon_\text{stat}}{(1-\gamma)^3}}+\frac{\sqrt{\varepsilon_\text{bias}}}{1-\gamma},
\end{equation}
where
\begin{equation}
	\label{eq:min}
	w^{(t)}_\star\in\argmin_{\norm{w}_2\leq \sqrt{K}W}L(w,\theta^{(t)},d^{(t)}),
\end{equation}
under the assumptions
\begin{align*}
\E\left[L(w^{(t)},\theta^{(t)},d^{(t)})- L(w^{(t)}_\star,\theta^{(t)},d^{(t)})|\theta^{(t)}\right] &\leq\varepsilon_\text{stat},\\
\E\left[L(w^{(t)}_\star,\theta^{(t)},d^\star)\right] &\leq\varepsilon_\text{bias}.
\end{align*}
In the original analysis in \cite{RN28}, $W$ is a parameter set by the user to control the norm of $w^{(t)}_\star$. In our setting, we normalize this parameter to $\sqrt{K}W$, which is analogous to requiring, for each agent $k$, the maximum norm of its optimal update $w^{(t)}_{k,\star}$ to be $W$. Not doing so would mean keeping a constant parameter $W$ despite increases in the dimensions, i.e.\ agents, of problem \eqref{eq:min}, incurring increases in the bias term. In the multi-agent setting, the iteration complexity given by the bound in \eqref{eq:gua} is worse by a factor $K$, compared to the single-agent setting. The curse of dimensionality appears when solving the minimization problem in \eqref{eq:min}, e.g.\ with gradient descent because the computation of exact gradients involves a sum/integral over $\S\times\A$, which has a dimension that grows exponentially with the number of agents. If we solve the minimization problem with stochastic projected gradient descent, then the problem of computing gradients disappears as we \emph{assume} access to samples to estimate gradients; however, the statistical guarantee becomes $O(K^2/\sqrt{N})$, from being $O(1/\sqrt{N})$ in the single-agent setting, due to the increase in dimensionality of the update $w$, in particular due to the scaling bound $\left\lVert w\right\rVert_2\leq \sqrt{K}W$ that is used by a classical convergence result of stochastic projected gradient descent \citep{RN70}. Implementing a sampler could, in turn, involve a curse of dimensionality. 

The dependencies of the minimization problem in \eqref{eq:min} cause the algorithm to incur additional scalability and implementability issues. As the projection step, the advantage function and the policy gradient depend on the states and actions of the entire network, which do not factorize, each agent would have to communicate to every other agent in the network at each iteration to solve the problem. As mentioned in the introduction, requiring such a level of communication is rarely viable in real-world applications in the decentralized MARL setting. 
\begin{remark}
	\label{rem:indip}
	These aforementioned problems do not arise in case of independent agents, where, for each agent $k$, the local transition probabilities satisfy $P_k(s'_k | s,a)=P_k(s'_k | s_k,a_k)$ and policies satisfy $\pi_{\theta_k}(a_k|s)=\pi_{\theta_k}(a_k|s_k)$. In this setting, as we show in Appendix \ref{app:indip}, it is possible to show that applying NPG to the whole network of $K$ agents corresponds to running $K$ independent runs of NPG applied to individual agents and to recover the same results of \cite{RN28} for the individual agents.
\end{remark}
\section{Decentralized NPG}
\label{main}
We design a decentralized version of NPG that is capable of exploiting the spatial decay of correlation properties that we assume and of avoiding the curse of dimensionality while still approximately converging to a globally optimal policy. We do so by limiting the communication range to agents that are at most at distance $r$ and defining, for each agent $k$, the localized advantage function, localized value function, localized Q-function, as follows: 
\begin{align*}
	\widetilde{A}_k^\pi(s_{N^{r}_k},a_{N^r_k}) &= \widetilde{Q}_k^\pi(s_{N^{r}_k},a_{N^r_k})-\widetilde{V}_k^\pi(s_{N^{r}_k}),\\
	\widetilde{V}^\pi_k(s_{N^{r}_k})&= \mathbb{E}\left[\sum_{t=0}^{\infty}\gamma^t r_k(s_k(t),a_k(t)) \big| \pi, s_{N^{r}_k}(0) = s_{N^{r}_k}\right],\\
	\widetilde{Q}^\pi_k(s_{N^{r}_k}, a_{N^r_k}) &=\E\left[\sum_{t=0}^\infty \gamma^t r_k(s_k(t), a_k(t))\big|\pi, s_{N^{r}_k}(0)=s_{N^{r}_k}, a_{N^r_k}(0)=a_{N^r_k}\right].
\end{align*}
Following Assumption \ref{ass:pol2}, we set $(\theta_k)_{N_{-k}^r}=0, \forall k\in\K$ and we never update these parameters, so that the policy of an agent and its gradient do not depend on the states of agents whose distance is greater than $r$. For each agent $k$, define the loss function
\begin{equation}
	\label{eq:loss}
	\widetilde{L}_k^r(w, \theta,\nu)=\E_{s,a \sim \nu}\left[\left(\widetilde{A}^{\pi_{\theta}}_k(s_{N^r_k},a_{N^r_k})-\nabla_{(\theta_k)_{N_{k}^r}}\log\pi_{\theta_k}(a_k|s_{N^r_k})\cdot w\right)^2\right].
\end{equation}
The minimization problem that each agent $k$ aims to solve at each step becomes
\begin{equation}
	\label{eq:app_update}
	w^\star\in\argmin_{\norm{w}_2\leq W}\widetilde{L}^r_k(w,\theta^{(t)},d^{(t)}).
\end{equation}
In Appendix \ref{app:stat} we show show how to solve this minimization problem in a decentralized manner and that, even if $d^{(t)}$ is a global distribution, it is possible to build a decentralized sampler of it assuming only access to a global clock.

\begin{algorithm}[h]
	\caption{Decentralized NPG}
	\label{alg:npg}
	\KwIn{Learning rate $\eta$; numbers of iterations $T$; an initialized policy $\pi^{(0)}$.}
	Set $(\theta_k)_{N_{-k}^r}=0, \forall k\in\K$\;
	\For{$t=0,\dots,T-1$; $k\in\K$}{
			Compute approximately $w^{(t)}_{k}\in\argmin_{\norm{w}_2\leq W}\widetilde{L}^r_k(w,\theta^{(t)},d^{(t)})$\;
			Compute the update $(\theta_k)_{N_{k}^r}^{(t+1)}= (\theta_k)_{N_{k}^r}^{(t)}+\frac{\eta}{1-\gamma}w^{(t)}_{k}.$
	}
\end{algorithm}
Exploiting decay of correlation properties and the policy design constraints in Assumption \ref{ass:pol2}, Algorithm \ref{alg:npg} removes the dependence on $K$ from the iteration complexity bound and addresses the curse of dimensionality and scalability and implementability issues outlined in Section \ref{sec:curse}.

\begin{theorem}
	\label{thm:main}
	Assume that Assumption \ref{ass:pol2} and Assumption \ref{ass:exp} hold. Assume that $\log\pi_\theta(a|s)$ is a $\delta$-smooth function of $\theta$ and that $\pi^{(0)}$ is the uniform distribution. Let $d^{(t)} = d^{\pi_t}_\nu(s,a)$ and $d^\star(s,a)= d^{\pi^\star}_\mu(s) \pi^\star(a|s)$. Let $\nu$ be a distribution of $s,a$ for which there exists $\kappa'\geq0$ such that
	 \[\max_{k\in\K}\sup_{w\in\R^d}\frac{w^\top\Sigma^{(t)}_{d^\star,k}w}{w^\top\Sigma^{(t)}_{\nu,k}w} \leq \kappa',\]
	where, $\forall \theta,\nu, k$
	\[\Sigma_{\nu,k}^\theta=\E_{s,a \sim \nu}\left[\nabla_{(\theta_k)_{N_{k}^r}}\log\pi_\theta(a_k|s_{N_k^r})(\nabla_{(\theta_k)_{N_{k}^r}}\log\pi_\theta(a_k|s_{N_k^r}))^\top\right]\]
	and $\Sigma_{\nu,k}^{(t)} \equiv \Sigma_{\nu,k}^{\theta^{(t)}}$. Let
	\begin{align*}
	\max_{k\in\K}\E\left[\widetilde{L}_k(w^{(t)}_{k,\star},\theta^{(t)},d^\star)\right] &\leq\varepsilon_\text{bias},\\
	\max_{k\in\K}\E\left[\widetilde{L}_k(w^{(t)}_{k},\theta^{(t)},d^{(t)})- \widetilde{L}_k(w^{(t)}_{k,\star},\theta^{(t)},d^{(t)})\big|\theta^{(t)}\right]&\leq\varepsilon_\text{stat},
	\end{align*}
	where
	\[w^{(t)}_{k,\star}\in\argmin_{\norm{w}_2\leq W}\widetilde{L}_k(w,\theta^{(t)},d^{(t)}).\]
	Then, Algorithm \ref{alg:npg}, with $\eta = \sqrt{2 \frac{\log|\A|}{(\delta K W^2T)}}$, has the following guarantee:
	\begin{equation}
		\label{eq:thm}
		\begin{aligned}
			\E\left[\min_{t<T}\{V^{\pi^\star}(\mu)-V^{(t)}(\mu)\}\right]&\leq \frac{W}{1-\gamma}\sqrt{\frac{2\delta\log\max_{k\in\K}|\A_k|}{T}}+\sqrt{\frac{\kappa'\varepsilon_\text{stat} }{(1-\gamma)^3}}+\frac{\sqrt{\varepsilon_\text{bias}}}{1-\gamma}\\
			&\quad+\underbrace{\frac{c\psi^{r+1}+c'\phi^{r+1}}{1-\gamma}}_\text{localization error}.
		\end{aligned}
	\end{equation}
\end{theorem}

Agent-wise, the assumptions in Theorem \ref{thm:main} correspond to the assumptions made in \cite{RN28}, in a setting where, for each agent $k$, the state space and the action space are $\S_{N_k^r}$ and $\A_k$, respectively, where the policy is defined as $\pi_\theta(a|s)= \pi_{\theta_k}(a_k|s_{N_k^r})$ and where the update $w^{(t)}_{k}$ is bounded by $W$. Theorem \ref{thm:main} shows that Decentralized NPG recovers the iteration complexity of the algorithm in \cite{RN28}, worsened only by the fourth term on the RHS of \eqref{eq:thm}. This localization error is exponentially small in $r$. Theorem \ref{thm:main} provides a dimension-free guarantee on the average expected cumulative rewards of the whole network, as the upper bound in \eqref{eq:thm} does not depend on the number of agents $K$ in the network, and it depends only on the logarithm of the cardinality of the action space of an individual agent, with no explicit dependence on the state space of agents.

Theorem \ref{thm:main} shows that, under the assumption on spatial decay of correlation, Decentralized NPG solves the curse of dimensionality and the scalability and implementability issues outlined in Section \ref{sec:curse}. The minimization problem in \eqref{eq:app_update} can be approximately solved in a decentralized manner through stochastic projected gradient descent, as we show in Appendix \ref{app:stat}, which leads to computational savings as we manage to eliminate the dependency on $K$ from the statistical guarantee of the algorithm, obtaining the same computational complexity of the single-agent setting, i.e.\ $O(1/\sqrt{N})$, where $N$ is the number of gradient steps, which is not surprising because problem \eqref{eq:app_update} regards only the advantage function and the policy of an individual agent. Then, using stochastic projected gradient descent and the sampler in Algorithm \ref{alg:sampler}, we recover, for each agent, the same expected sample complexity of the single agent setting ($2NT/(1-\gamma)$, where $2/(1-\gamma)$ is the expected length of a sampler episode). Decentralized NPG can be run locally by each agent and only requires information from neighbors within distance $r$.

The role of the term $\varepsilon_{\text{bias}}$ in \eqref{eq:thm} has a difference from the role that $\varepsilon_{\text{bias}}$ has in \eqref{eq:agar} \citep{RN28}. They both represent the worst-case error that is made by the agents when they approximate their current advantage function with a linear combination of the elements of the gradient of their current policy and encode the \textit{transfer} error that we make shifting the distribution to $d^\star$. In addition to that, $\varepsilon_{\text{bias}}$ in \eqref{eq:thm} also encodes the localization error that we make in Algorithm \ref{alg:npg} by using the localized loss defined in \eqref{eq:loss}. In Appendix \ref{app:bias} we give a bound for this localization error of the bias term, showing that the localized bias is at most the non-localized bias, i.e. the bias associated with an infinite range parameter $r$, plus a quantity that decreases to $0$ exponentially fast in $r$.

\begin{remark}
	\label{rem:harv}
	The works in \citep{RN11, RN26, RN30, RN31} are closely related to our contribution, as they also use decay of correlation assumptions to provably avoid the curse of dimensionality in MARL. Our contribution differs from these works in the following main ways:
	\begin{enumerate}
		\item (Decay of correlation) We consider a more general version of the decay of correlation property (Assumption \ref{ass:dobr}) and, differently from them, we also require a decay of correlation property to hold for the policy (Assumption \ref{ass:pol}). Assumption \ref{ass:dobr} recovers the version they consider in \cite{RN26} in the case $\beta = 0$ and $\gamma = 1$. The generality of our condition allows us to consider transition dynamics that are not truncated, as they do, and to control the truncation of the policy.
		\item (Methodology) Our method is based on NPG framework, while their methods is based on policy gradient and actor-critic methods.
		\item (Optimality) We present statistical error bounds w.r.t.\ to the \emph{optimal} policy, while the bounds they give are w.r.t.\ a stationary policy.
		\item (Computational complexity) Our method has a computational complexity that does not depend, for any agent $k$, on the number of agents $K$ or the number of neighbors $|N_k^r|$. The method in \citep{RN31} is shown to have a computational complexity that scales as $O(\log|\S||\A|)$, hence depending linearly on $K$, using additional assumptions on the minimum local exploration. The method in \citep{RN30} is shown to have computational complexity that scales as $O(\log\max_{k\in\K}|\S_{N_k^r}||\A_{N_k^r}|)$, hence depending linearly on $|N_k^r|$, using additional assumptions on the stationarity and on the mixing rates of the MDP.
		\item (Statistical/Iteration Complexity) Under the only assumptions on decay of correlation and local policy, our method has an iteration complexity that scales as $O(\sqrt{\log\max_{k\in\K}|\A_k|})$. The methods in \citep{RN31, RN30} have an iteration complexity that does not depend on the state or action spaces.
	\end{enumerate}
\end{remark}

\section{Conclusion}
\label{conclu}
We have investigated applications of the NPG framework to MARL, showing how a standard assumption on the spatial decay of correlation for the dynamics and for the policy on a network of agents, expressed through a form of Dobrushin condition, induces a form of exponential decay in the cumulative rewards that can be exploited by a localized version of NPG to avoid the curse of dimensionality. The version of NPG that we design scales to large networks and yields convergence guarantees to the optimal policy that are analogous to the ones in \cite{RN28}, worsened only by a localization error that decreases exponentially with the communication radius. Our analysis does not consider regularization, which has been shown to accelerate convergence for NPG methods \citep{RN77} and yield linear convergence rates \citep{RN29}.

\bibliographystyle{plainnat}
\bibliography{References}
\newpage
\appendix
\section{Policy Class Example}
\label{app:pol_ex}

Let $\widetilde{r}=\max_{k,k'\in\K}d(k,k')$ be the maximum distance between two agents. Define a set of parameterized differentiable functions $\{f_{(\theta_k)_r}:\mathcal{S}_{N^r_{k}}\times\mathcal{A}_k\rightarrow\mathcal{C}|0\leq r\leq \widetilde{r}\}$, where $\mathcal{C}\subset[-C,C]$ and $C>0$, a set of parameters $\{\alpha_r\geq 0|0\leq r\leq \widetilde{r}\}$ and let, for each agent $k$,
\[f_{\theta_k}(s,a_k)=\sum_{r=0}^{\widetilde{r}} \alpha_rf_{(\theta_k)_r}(s_{N^r_{k}},a_k),\]
\[\pi_{\theta_k}(a_k|s)=\frac{\exp(f_{\theta_k}(s,a_k))}{\sum_{a' \in \mathcal{A}_k}\exp(f_{\theta_k}(s,a'))}.\]
By tuning the parameters $\alpha_r$, we can make any policy belonging to this policy class respect Assumptions \ref{ass:pol} and \ref{ass:pol2}, as we show in the following. Let $r\in\{0,\dots,\widetilde{r}\}$, let $s, \widetilde{s} \in \mathcal{S}$ be such that $s_{N^r_{k}}=\widetilde{s}_{N^r_{k}}$, then
\begin{align*}
	&TV(\pi_{\theta_k}(\cdot|s),\pi_{\theta_k}(\cdot|\widetilde{s}))=\frac{1}{2}\sum_{a \in \mathcal{A}_k}\left|\pi_{\theta_k}(a|s)-\pi_{\theta_k}(a|\widetilde{s})\right|\\
	&=\frac{1}{2}\sum_{a \in \mathcal{A}_k}\left|\frac{\exp(f_{\theta_k}(s,a))}{\sum_{a' \in \mathcal{A}_k}\exp(f_{\theta_k}(s,a'))}- \frac{\exp(f_{\theta_k}(\widetilde{s},a))}{\sum_{a' \in \mathcal{A}_k}\exp(f_{\theta_k}(\widetilde{s},a'))}\right|\\
	&=\frac{\sum_{a \in \mathcal{A}_k}\left|\sum_{a' \in \mathcal{A}_k}\exp(f_{\theta_k}(s,a))\exp(f_{\theta_k}(\widetilde{s},a'))- \exp(f_{\theta_k}(\widetilde{s},a))\exp(f_{\theta_k}(s,a'))\right|}{2\sum_{a' \in \mathcal{A}_k}\exp(f_{\theta_k}(\widetilde{s},a'))\sum_{a' \in \mathcal{A}_k}\exp(f_{\theta_k}(s,a'))}\\
	&\leq \frac{\sum_{a \in \mathcal{A}_k}\sum_{a' \in \mathcal{A}_k}\left|\exp(f_{\theta_k}(s,a))\exp(f_{\theta_k}(\widetilde{s},a'))- \exp(f_{\theta_k}(\widetilde{s},a))\exp(f_{\theta_k}(s,a'))\right|}{2\sum_{a' \in \mathcal{A}_k}\exp(f_{\theta_k}(\widetilde{s},a'))\sum_{a' \in \mathcal{A}_k}\exp(f_{\theta_k}(s,a'))}\\
	&\leq \frac{\sum_{a \in \mathcal{A}_k}\left|\exp(f_{\theta_k}(\widetilde{s},a))- \exp(f_{\theta_k}(s,a))\right|}{\sum_{a \in \mathcal{A}_k}\exp(f_{\theta_k}(\widetilde{s},a))}\\
	&\leq \frac{\sum_{a \in \mathcal{A}_k}\left|f_{\theta_k}(\widetilde{s},a)- f_{\theta_k}(s,a)\right|\exp(\sup_{s'\in \{s, \widetilde{s}\}}f_{\theta_k}(s',a))}{\sum_{a \in \mathcal{A}_k}\exp(f_{\theta_k}(\widetilde{s},a))}\\
	&\leq e^{2C(\widetilde{r}-r)}\frac{\sum_{a \in \mathcal{A}_k}\left|f_{\theta_k}(\widetilde{s},a)- f_{\theta_k}(s,a)\right|\exp(f_{\theta_k}(\widetilde{s},a))}{\sum_{a \in \mathcal{A}_k}\exp(f_{\theta_k}(\widetilde{s},a))}\\
	& = e^{2C(\widetilde{r}-r)}\E_{\pi_{\theta_k}}\left|\sum_{r'=r+1}^{\widetilde{r}} \alpha_{r'}\left(f_{(\theta_k)_{r'}}(\widetilde{s}_{N^{r'}_{k}},a)- f_{(\theta_k)_{r'}}(s_{N^{r'}_{k}},a)\right)\right|
\end{align*}
\begin{align*}
	& \leq e^{2C(\widetilde{r}-r)}\sum_{r'=r+1}^{\widetilde{r}} \alpha_{r'}\E_{\pi_{\theta_k}}\left|f_{(\theta_k)_{r'}}(\widetilde{s}_{N^{r'}_{k}},a)- f_{(\theta_k)_{r'}}(s_{N^{r'}_{k}},a)\right|\\
	& \leq 2Ce^{2C(\widetilde{r}-r)}\sum_{r'=r+1}^{\widetilde{r}} \alpha_{r'}.
\end{align*}
Setting the parameters $\{\alpha_{r'}\}_{r'\in\{r+1,\dots,\widetilde{r}\}}$ small enough ensures that the policy respects Assumption \ref{ass:pol}. Similarly, Assumption \ref{ass:pol2} is satisfied for a value $r$ of the range parameter if $\alpha_{r'}=0$ $\forall r'\in\{r+1,\dots,\widetilde{r}\}$.

\section{Proof of Proposition \ref{prop:decay}}
\label{app:exp_decay}
\subsection{Preliminary Lemmas}
To prove Proposition \ref{prop:decay}, we need a series of intermediate results, which we state and prove for completeness. Results similar to Lemmas \ref{lemma} and \ref{lemma:3} can be found in Chapter 8 of \cite{RN75}, Lemma \ref{lemma:2} is an extension of results from \cite{RN26}.

\begin{lemma}
	\label{lemma}
	Let $f:\Z\rightarrow[m,M]$, where $\Z=\prod_{k \in \K}\Z_k$ and $m,M\in\R$. For every $k\in\K$, let $\mu_k$ and $\nu_k$ be two distributions on $\Z_k$. Let $\mu$ and $\nu$ be the respective product distributions. Let $\delta_k(f(z))=\sup_{z_k,z_{-k},z'_k}\left|f(z_k,z_{-k})-f(z'_k,z_{-k})\right|$. Then:
	
	\[\left|\E_{z\sim \mu}f(z)-\E_{z\sim \nu}f(z)\right|\leq\sum_{k \in \K}TV(\mu_k,\nu_k)\delta_k(f).\]
	
\end{lemma}
\begin{proof}
	We prove Lemma \ref{lemma} by induction. Note that \[TV(\mu,\nu)=\frac{1}{2}\max_{|h|\leq 1}\left|\E_\mu(h)-\E_\nu(h)\right|\] is an equivalent formulation of the total variation distance \citep{RN76}. For $|\K|=1$, we have that
	\begin{align*}
		\left|\E_{\mu_1}(f)-\E_{\nu_1}(f)\right|&=\left|\E_{\mu_1}\left(f-\frac{M+m}{2}\right)-\E_{\nu_1}\left(f-\frac{M+m}{2}\right)\right|\\
		&=\frac{M-m}{2}\left|\E_{\mu_1}\left(\frac{2f}{M-m}-\frac{M+m}{M-m}\right)-\E_{\nu_1}\left(\frac{2f}{M-m}-\frac{M+m}{M-m}\right)\right|
	\end{align*}
\begin{align*}
		&\leq \frac{M-m}{2}\max_{|h|\leq 1}\left|\E_{\mu_1}(h)-\E_{\nu_1}(h)\right|\\
		&=TV({\mu_1},{\nu_1})\delta_1(f).
	\end{align*}
	As induction assumption, assume that Lemma \ref{lemma} holds for $|\K|-1$. Then:
	\begin{align*}
		\left|\E_{z\sim\mu}f(z)-\E_{z\sim\nu}f(z)\right|&=\left|\E_{z_1\sim\mu_1}\E_{z_{2:n}\sim\mu_{2:n}}f(z)-\E_{z_1\sim\nu_1}\E_{z_{2:n}\sim\nu_{2:n}}f(z)\right|\\
		&\leq \left|\E_{z_1\sim\mu_1}\E_{z_{2:n}\sim\mu_{2:n}}f(z)-\E_{z_1\sim\mu_1}\E_{z_{2:n}\sim\nu_{2:n}}f(z)\right| \\
		&\quad +\left|\E_{z_1\sim\mu_1}\E_{z_{2:n}\sim\nu_{2:n}}f(z)-\E_{z_1\sim\nu_1}\E_{z_{2:n}\sim\nu_{2:n}}f(z)\right|\\
		&\leq\E_{z_1\sim\mu_1} \left|\E_{z_{2:n}\sim\mu_{2:n}}f(z)-\E_{z_{2:n}\sim\nu_{2:n}}f(z)\right| \\
		&\quad +\left|\E_{z_1\sim\mu_1}\widetilde{f}(z_1)-\E_{z_1\sim\nu_1}\widetilde{f}(z_1)\right|.
	\end{align*}
	where $\widetilde{f}(z_1)=\E_{z_{2:n}\sim\nu_{2:n}}f(z)$. By induction assumption:
	
	\[\left|\E_{z_{2:n}\sim\mu_{2:n}}f(z)-\E_{z_{2:n}\sim\nu_{2:n}}f(z)\right|\leq\sum_{k\neq 1 \in \K}TV(\mu_k,\nu_k)\delta_k(f(z_1,\cdot))\leq\sum_{k\neq 1 \in \K}TV(\mu_k,\nu_k)\delta_k(f).\]
	
	Since
	\begin{align*}
		\delta_1(\widetilde{f})&=\sup_{z_1,z_1'}\left|\E_{z_{2:n}\sim\nu_{2:n}}f(z_1, z_{2:n}) - \E_{z_{2:n} \sim\nu_{2:n}}f(z_1', z_{2:n})\right|\\
		&\leq\sup_{z_1,z_1'}\E_{z_{2:n}\sim\nu_{2:n}}\left|f(z_1, z_{2:n}) - f(z_1', z_{2:n})\right|\\
		&\leq\sup_{z_1,z_1'}\sup_{z_{2:n}}\left|f(z_1, z_{2:n}) - f(z_1', z_{2:n})\right|=\delta_1(f),
	\end{align*}
	we have 
	\begin{align*}
		\left|\E_{z\sim\mu}f(z)-\E_{z\sim\nu}f(z)\right|&\leq \E_{z_1\sim\mu_1}\sum_{k\neq 1 \in \K}TV(\mu_k,\nu_k)\delta_k(f) + TV({\mu_1},{\nu_1})\delta_1(f)\\
		&\leq \sum_{k \in \K}TV(\mu_k,\nu_k)\delta_k(f),
	\end{align*}
which concludes the induction.
\end{proof}

\begin{lemma}
	\label{lemma:3}
	Consider a Markov Chain with state $z\in\Z$,  where $\Z=\prod_{k \in \K}\Z_k$ and $\K$ is defined as in Section \ref{setting}. Suppose its transition probability factorizes as
	\[P(z(t+1)|z(t))=\prod_{k\in \K}P_k(z_k(t+1)|z(t)).\]
	Let $C\in\R^{K\times K}$ be a matrix whose elements respect the condition \[C_{ij}\geq\sup_{z_j,z_{-j},z'_j} TV(P_i(\cdot|z_j,z_{-j}),P_i(\cdot|z'_j,z_{-j})).\] If $\sum_{j \in \K}e^{\beta d(j,k)}C_{kj}\leq\rho$, then, $\forall \J\subseteq\K$,
	\[\sup_{z_j,z_{-j},z'_j}TV(P_i(\cdot|z_J,z_{-J}),P_i(\cdot|z'_J,z_{-J}))\leq \sum_{j \in \J}C_{ij}\]
	and 
	\[\sup_{z_J,z_{-J},z'_J}TV(P_i(\cdot|z_J,z_{-J}),P_i(\cdot|z'_J,z_{-J}))\leq \rho e^{-\beta d(\J,i)},\]
	where $d(\J,i)=\min_{j \in \J} d(j,i)$.
\end{lemma}
\begin{proof}
	We prove the first claim of Lemma \ref{lemma:3} by induction. The first claim clearly holds if $|\J|=1$. As induction assumption, assume that the first claim holds for a generic $\J$. Then, it holds for $\J'=\J + \{k\}$:
	\begin{align*}
		&\sup_{z_j,z_{-j},z'_j}TV(P_i(\cdot|z_{J'},z_{-J'}),P_i(\cdot|z'_{J'},z_{-J'}))= \sup_{\substack{A\subseteq\Z_i\\z_j,z_{-j},z'_j}}\left|P_i(A|z_{J'},z_{-J'})-P_i(A|z'_{J'},z_{-J'})\right|\\
		&\leq \sup_{\substack{A\subseteq\Z_i\\z_j,z_{-j},z'_j}}\left|P_i(A|z_{J'},z_{-J'})-P_i(A|z'_{J},z_{-J})\right|+ \sup_{\substack{A\subseteq\Z_i\\z_j,z_{-j},z'_j}}\left|P_i(A|z'_{J},z_{-J})-P_i(A|z'_{J'},z_{-J'})\right|\\
		&\leq\sum_{j \in \J}C_{ij} + C_{ik}= \sum_{j \in \J'}C_{ij}.
	\end{align*}
	
	The second claim follows immediately, since
	\[e^{\beta d(\J,i)}\sum_{j \in \J}C_{ij}\leq\sum_{j \in \J}e^{\beta d(j,i)}C_{ij}\leq\sum_{j \in \K}e^{\beta d(j,i)}C_{ij}\leq\rho,\]
	and
	\[\sum_{j \in \J}C_{ij}\leq\rho e^{-\beta d(\J,i)}.\]
\end{proof}

\begin{lemma}
	\label{lemma:2}
	Consider the setting of Lemma \ref{lemma:3}. For a generic value of $r$, denote by $d_{t}$ and $\widetilde{d}_{t}$ the distribution of $z(t)$ with starting state, respectively, $z=(z_{N_k^r},z_{N_{-k}^r})$ and $\widetilde{z}=(z_{N_{k}^r},\widetilde{z}_{N_{-k}^r})$. Then, if $\sum_{j \in \K}e^{\beta d(j,k)}C_{kj}\leq\rho$, we have that $TV(d_{t,k},\widetilde{d}_{t,k})\leq \rho^t e^{-\beta r}$, $\forall k \in\K$. 
\end{lemma}
\begin{proof}
	We prove Lemma \ref{lemma:2} by induction. The case where $t=1$ follows from Lemma \ref{lemma:3}. As induction assumption, assume that Lemma \ref{lemma:2} holds for $t$. Then,
	\begin{align*}
		&\left|\E_{s\sim d_{t+1,k}(s)}\1_A(s) -\E_{s\sim\widetilde{d}_{t+1,k}}\1_A(s)\right|\\
		&=\left|\E_{z\sim d_{t}}E_{s\sim P_k(\cdot|z)}\1_A(s)-\E_{z\sim \widetilde{d}_{t}}\E_{s\sim P_k(\cdot|z)}\1_A(s)\right|\\
		&\leq \sum_{j \in \K}TV(d_{t,j},\widetilde{d}_{t,j})\delta_j(E_{s\sim P_k(\cdot|\cdot)}\1_A(s))\\
		&\leq \sum_{j \in \K}TV(d_{t,j},\widetilde{d}_{t,j})C_{kj}\\
		&\leq \sum_{j \in \K} \rho^te^{-\beta(r-d(j,k))}C_{kj}\\
		&=  \rho^te^{-\beta r}\sum_{j \in \K}e^{\beta d(j,k)}C_{kj}\leq \rho^{t+1} e^{-\beta r},
	\end{align*}
	where we used Lemma \ref{lemma} in the first inequality.
\end{proof}

\begin{lemma}
	\label{lemma:value}
	Consider the setting of Lemma \ref{lemma:3}. Let $P^t(z'|z)= P(z(t)=z'|z(0)=z)$ and
	\[\delta_iP^t_k=\sup_{z_i,z_{-i},z'_i}TV(P_k^t(\cdot|z_i,z_{-i}),P_k^t(\cdot|z'_i,z_{-i})).\]
	If $\sum_{j \in \K}e^{\beta d(j,k)}C_{kj}\leq\rho$, then $\forall k \in\K$
	\[\sum_{i \in \K}e^{\beta d(k,i)}\delta_iP^t_k\leq\rho^t.\]
\end{lemma}
\begin{proof}
	We prove Lemma \ref{lemma:value} by induction. The claim holds for $t=1$:
	\[\sum_{i \in \K}e^{\beta d(k,i)}\delta_iP^t_k=\sum_{i \in \K}e^{\beta d(k,i)}C_{k,i}\leq\rho\]
	As induction assumption, we assume that the claim holds for $t$. Then, using Lemma \ref{lemma},
	\begin{align*}
		\delta_iP^{t+1}_k&=\sup_{\substack{A\subseteq\S_k\\z_i,z_{-i},z'_i}}\left|\E_{s\sim P_k^{t+1}(\cdot|z_i,z_{-i})}\1_A(s) -\E_{s\sim P_k^{t+1}(\cdot|z'_i,z_{-i})}\1_A(s)\right|\\
		& =\sup_{\substack{A\subseteq\S_k\\z_i,z_{-i},z'_i}}\left|\E_{x\sim P^{t}(\cdot|z_i,z_{-i})}E_{s\sim P_k(\cdot|x)}\1_A(s)-\E_{x\sim P^{t}(\cdot|z'_i,z_{-i})}\E_{s\sim P_k(\cdot|x)}\1_A(s)\right|\\
		&\leq \sup_{z_i,z_{-i},z'_i}\sum_{j \in \K}TV(P^{t}_j(\cdot|z_i,z_{-i}),P^{t}_j(\cdot|z'_i,z_{-i}))\delta_j(E_{s\sim P_k(\cdot|\cdot)}\1_A(s))\\
		&\leq \sum_{j \in \K}\delta_iP^t_jC_{kj}
	\end{align*}
	and, using the inverse triangle inequality,
	\begin{align*}
		\sum_{i \in \K}e^{\beta d(k,i)}\delta_iP^{t+1}_k&\leq\sum_{i \in \K}e^{\beta d(k,i)}\sum_{j \in \K}\delta_iP^t_jC_{kj}\\
		&\leq\sum_{j \in \K}e^{\beta d(k,j)}C_{kj}\sum_{i \in \K}e^{\beta(d(k,i)- d(k,j))}\delta_iP^t_j\\
		&\leq\sum_{j \in \K}e^{\beta d(k,j)}C_{kj}\sum_{i \in \K}e^{\beta d(j,i)}\delta_iP^t_j\leq\rho^{t+1},
	\end{align*}
	which concludes the induction.
\end{proof}

\subsection{Main Result}
\begin{proof}[of Proposition \ref{prop:decay}]
	The following holds for every $k \in\K$. Let $s,\widetilde{s}\in\S$, $a,\widetilde{a}\in\A$ be such that $s_{N^r_k} = \widetilde{s}_{N^r_k}$ and $a_{N^r_k} = \widetilde{a}_{N^r_k}$. Notice that
	\begin{align*}
		&\left|Q_k^\pi(s,a)-Q_k^\pi(\widetilde{s},\widetilde{a})\right|\\
		&\leq\sum_{t=0}^\infty\gamma^t\left|\E\left[ r_k(s_k(t), a_k(t))\big|\pi, s(0)=s, a(0)=a\right]-\E\left[  r_k(s_k(t), a_k(t))\big|\pi, s(0)=\widetilde{s}, a(0)=\widetilde{a}\right]\right|\\
		&\leq\sum_{t=1}^\infty\gamma^t\left|\E\left[ r_k(s_k(t), a_k(t))\big|\pi, s(0)=s, a(0)=a\right]-\E\left[  r_k(s_k(t), a_k(t))\big|\pi, s(0)=\widetilde{s}, a(0)=\widetilde{a}\right]\right|\\
		&\leq\sum_{t=1}^\infty\gamma^tTV(d_{t,k},\widetilde{d}_{t,k}),
	\end{align*}
	where $d_{t,k}(s_k,a_k)$ and $\widetilde{d}_{t,k}(s_k,a_k)$ are the distributions of $s_k,a_k$ at time $t$ with starting point $(s,a)$ and $(\widetilde{s},\widetilde{a})$, respectively. We use Lemma \ref{lemma:2} to bound $TV(d_{t,k},\widetilde{d}_{t,k})$. The structure of our MDP implies that:
	\[P(s(t+1),a(t+1)|s(t),a(t))=\prod_{k\in \K}\pi^k(a_k(t+1)|s(t+1))P_k(s_k(t+1)|s(t),a(t)).\]
	Let $C$ be defined as in Assumption \ref{ass:dobr} and note that
	\begin{align*}
		C_{kj}&=\sup_{s_j,s_{-j},a_j,a_{-j}, s'_j,a'_j} TV(P_k(\cdot|s_j,s_{-j},a_j,a_{-j}),P_k(\cdot|s'_j,s_{-j},a'_j,a_{-j}))\\&\geq\sup_{s_j,s_{-j},a_j,a_{-j}, s'_j,a'_j} TV(P_k(\cdot,\cdot|s_j,s_{-j},a_j,a_{-j}),P_k(\cdot,\cdot|s'_j,s_{-j},a'_j,a_{-j})).
	\end{align*}
	Then, if Assumption \ref{ass:dobr} holds, the requirements of Lemma \ref{lemma:2} are satisfied. Therefore, $TV(d_{t,k},\widetilde{d}_{t,k})\leq \rho^t e^{-\beta r}$ and
	\[	\left|Q_k^\pi(s,a)-Q_k^\pi(\widetilde{s},\widetilde{a})\right|\leq\sum_{t=1}^\infty\gamma^tTV(d_{t,k},\widetilde{d}_{t,k})\leq e^{-\beta r}\sum_{t=1}^\infty\gamma^t\rho^t=\frac{\gamma\rho e^{-\beta r}}{1-\gamma\rho}.\]
	
	We use this result to prove the exponential decay property for the value function. Let
	\[\delta_jQ_k^\pi(s,a)=\sup_{s_j,s_{-j},a_j,a_{-j}, s'_j,a'_j}|Q_k^\pi(s_j,s_{-j},a_j,a_{-j})-Q_k^\pi(s'_j,s_{-j},a'_j,a_{-j})|.\]
	Using Lemma \ref{lemma} and Assumption \ref{ass:pol}, we have that	
	\begin{align*}
		&\left|V_k^\pi(s)-V_k^\pi(\widetilde{s})\right|=\left|\E_{a \sim \pi(\cdot|s)}Q_k^\pi(s,a)-\E_{a \sim \pi(\cdot|\widetilde{s})}Q_k^\pi(\widetilde{s},a)\right|\\
		&\leq \left|\E_{a \sim \pi(\cdot|s)}Q_k^\pi(s,a)-\E_{a \sim \pi(\cdot|\widetilde{s})}Q_k^\pi(s,a)\right| +\left|\E_{a \sim \pi(\cdot|\widetilde{s})}Q_k^\pi(s,a)-\E_{a \sim \pi(\cdot|\widetilde{s})}Q_k^\pi(\widetilde{s},a)\right|\\
		&\leq \sum_{i\in\K}TV(\pi_i(\cdot|s),\pi_i(\cdot|\widetilde{s}))\delta_iQ_k^\pi(s,a)+\frac{\gamma\rho e^{-\beta r}}{1-\gamma\rho}\\
		&\leq \xi e^{-\beta r}\sum_{i\in \K} e^{-\beta d(i,k)}\delta_iQ_k^\pi(s,a) +\frac{\gamma\rho e^{-\beta r}}{1-\gamma\rho}.
	\end{align*}
	We have already shown that the MDP satisfies the condition of Lemma \ref{lemma:value}. Using it we obtain
	\[\sum_{i \in \K}e^{\beta d(k,i)}\delta_i(Q^\pi_k(s,\cdot))\leq \sum_{t=1}^\infty\gamma^t\sum_{i \in \K}e^{\beta d(k,i)}\delta_iP^t_k\leq \sum_{t=1}^\infty\gamma^t\rho^t=\frac{\gamma\rho}{1-\gamma\rho},\]
	where
	\[\delta_iP^t_k=\sup_{s_j,s_{-j},a_j,a_{-j}, s'_j,a'_j}TV(P_k^t(\cdot,\cdot|s_j,s_{-j},a_j,a_{-j}),P_k^t(\cdot,\cdot|s'_j,s_{-j},a'_j,a_{-j})).\]
	Then, we have that
	\[\left|V_k^\pi(s)-V_k^\pi(\widetilde{s})\right|\leq\frac{\gamma\rho (1+\xi) e^{-\beta r}}{1-\gamma\rho}.\]
\end{proof}

\section{Decay for the Optimal Policy}
\label{app:ass}
\begin{lemma}
	\label{lemma:4}
	Assume that the exponential decay property holds for the Q-function with parameters $(c,\psi)$. Then the exponential decay property holds also for the value function associated with the optimal policy, with parameters $(3c,\psi)$.
\end{lemma}
\begin{proof}
	The following holds for every $k \in\K$. Let $s,\widetilde{s}\in\S$ be such that $s_{N^r_k} = \widetilde{s}_{N^r_k}$ and let $a_{N^r_{-k}}\in\A_{N^r_{-k}}$.
	\begin{align*}
		&\left|V_k^\star(s)-V_k^\star(\widetilde{s})\right|=\left|\E_{a \sim \pi^\star(\cdot|s)}Q_k^\star(s,a)-\E_{a \sim \pi^\star(\cdot|\widetilde{s})}Q_k^\star(\widetilde{s},a)\right|=\left|\max_aQ_k^\star(s,a)-\max_aQ_k^\star(\widetilde{s},a)\right|\\
		&=\left|\max_aQ_k^\star(s,a)-\max_{a_{N^r_k}}Q_k^\star(s,a_{N^r_k},a_{N^r_{-k}})+\max_{a_{N^r_k}}Q_k^\star(s,a_{N^r_k},a_{N^r_{-k}})\right.\\&\quad-\left.\max_aQ_k^\star(\widetilde{s},a)-\max_{a_{N^r_k}}Q_k^\star(\widetilde{s},a_{N^r_k},a_{N^r_{-k}})+\max_{a_{N^r_k}}Q_k^\star(\widetilde{s},a_{N^r_k},a_{N^r_{-k}})\right|\\
		&=\left|\max_{a_{N^r_k}}Q_k^\star(s,a_{N^r_k},a_{N^r_{-k}})-\max_{a_{N^r_k}}Q_k^\star(\widetilde{s},a_{N^r_k},a_{N^r_{-k}})\right|\\
		&\quad+\left|\max_aQ_k^\star(\widetilde{s},a)-\max_{a_{N^r_k}}Q_k^\star(\widetilde{s},a_{N^r_k},a_{N^r_{-k}})\right|+\left|\max_aQ_k^\star(s,a)-\max_{a_{N^r_k}}Q_k^\star(s,a_{N^r_k},a_{N^r_{-k}})\right|\\
		&\leq \left|\max_{a_{N^r_k}}Q_k^\star(s,a_{N^r_k},a_{N^r_{-k}})-\max_{a_{N^r_k}}Q_k^\star(\widetilde{s},a_{N^r_k},a_{N^r_{-k}})\right| +2c\psi^{r+1}.
	\end{align*}
	Let $a'_{N^r_k}\in\argmax_{a_{N^r_k}}Q_k^\star(s,a_{N^r_k},a_{N^r_{-k}})$, then
	\begin{align*}
		&Q_k^\star(s,a'_{N^r_k},a_{N^r_{-k}})-\max_{a_{N^r_k}}Q_k^\star(\widetilde{s},a_{N^r_k},a_{N^r_{-k}})\\
		&\leq Q_k^\star(\widetilde{s},a'_{N^r_k},a_{N^r_{-k}})-\max_{a_{N^r_k}}Q_k^\star(\widetilde{s},a_{N^r_k},a_{N^r_{-k}})+c\psi^{r+1}
			\end{align*}
\begin{align*}
		&\leq\max_{a_{N^r_k}}Q_k^\star(\widetilde{s},a_{N^r_k},a_{N^r_{-k}})-\max_{a_{N^r_k}}Q_k^\star(\widetilde{s},a_{N^r_k},a_{N^r_{-k}})+c\psi^{r+1}=c\psi^{r+1}.
	\end{align*}
	The same holds for $\max_{a_{N^r_k}}Q_k^\star(\widetilde{s},a_{N^r_k},a_{N^r_{-k}})-\max_{a_{N^r_k}}Q_k^\star(s,a_{N^r_k},a_{N^r_{-k}})$. The lemma follows immediately.
\end{proof}

We make an assumption on the minimum influence that the action of an agent has on its expected future rewards. Assumption \ref{ass:mininfl} and Proposition \ref{prop:optdecay} hold for any $k\in\K$.
\begin{assumption}
	\label{ass:mininfl}
	Let $\A^\star = \argmax_{a\in\A}Q_k^\star(s,a)$, $\forall s\in\S$. Assume that, if $\widetilde{a}$ is such that $\widetilde{a}_k\notin\A^\star_k$, then
	\[\left|Q_k^\star(s,a)-Q_k^\star(s,\widetilde{a})\right|\geq R\]
\end{assumption}

\begin{proposition}
	\label{prop:optdecay}
	Assume that the exponential decay property holds for the Q-function with parameters $(c,\psi)$ and that Assumption \ref{ass:mininfl} holds. Let $s,\widetilde{s}\in\S$ be such that $s_{N^r_k} = \widetilde{s}_{N^r_k}$. Let $\A^\star = \argmax_{a\in\A}Q_k^\star(s,a)$ and $\widetilde{a}\in\argmax_{a\in\A}Q_k^\star(\widetilde{s},a)$. If $r>\log_\psi R/4c$, then $\widetilde{a}_k\in\A^\star_k$.
\end{proposition}
\begin{proof}
	We prove this by contradiction. Lemma \ref{lemma:4} shows that, $\forall r>0$, if $s,\widetilde{s}\in\S$ are such that $s_{N^r_k} = \widetilde{s}_{N^r_k}$,
	\[\left|V_k^\star(s)-V_k^\star(\widetilde{s})\right|=\left|\max_aQ_k^\star(s,a)-\max_aQ_k^\star(\widetilde{s},a)\right|\leq 3c\psi^{r+1}.\]
	Let $a\in\argmax_{a\in\A}Q_k^\star(s,a)$ and $\widetilde{a}\in\argmax_{a\in\A}Q_k^\star(\widetilde{s},a)$. Let $\A^\star = \argmax_{a\in\A}Q_k^\star(s,a)$ and assume that $\widetilde{a}_k\notin\A^\star_k$. Then
	\begin{align*}
		\left|Q_k^\star(s,a)-Q_k^\star(\widetilde{s},\widetilde{a})\right|		&=\left|Q_k^\star(s,a)-Q_k^\star(s,\widetilde{a})+Q_k^\star(s,\widetilde{a})-Q_k^\star(\widetilde{s},\widetilde{a})\right|\\
		&\geq\left|\left|Q_k^\star(s,a)-Q_k^\star(s,\widetilde{a})\right|-\left|Q_k^\star(s,\widetilde{a})-Q_k^\star(\widetilde{s},\widetilde{a})\right|\right|\\
		&\geq\left|Q_k^\star(s,a)-Q_k^\star(s,\widetilde{a})\right|-c\psi^{r+1}\geq R-c\psi^{r+1}
	\end{align*}
	where we used Assumption \ref{ass:mininfl} in the last passage. Then, due to Lemma \ref{lemma:4}, if $r>\log_\psi R/4c$ we have a contradiction. 
\end{proof}

Proposition \ref{prop:optdecay} shows that the optimal policy of an agent is not influenced by distant agents. Assumption \ref{ass:pol} ensures that the policy class we consider respects this condition.

\section{Independent Agents}
\label{app:indip}
By completely independent agents we mean agents whose transition dynamics are independent and whose policy is defined, for $s\in\S,a\in\A$, as:
\begin{align*}
	\pi_\theta(a|s)&= \prod_{k \in \K}\pi_{\theta_k}(a_k| s_k)=\prod_{k}\frac{e^{f_{\theta_k}(s_k, a_k)}}{\sum_{a' \in \A_k}e^{f_{\theta_k}(s_k,a')}}.
\end{align*}
In accordance to previous assumptions, we also assume that $\pi_{\theta_k}(a_k| s_k)$ is $\delta$-smooth.

\begin{proof}[of Remark \ref{rem:indip}]
	Firstly we show that the two applications of the algorithm coincide. Let
	\[L(w)= \E_{s \sim d_{\nu}^{\pi},a \sim \pi_{\theta}(\cdot|s)}\left[\left(A^{\pi_{\theta}}(s,a)-w\cdot\nabla_{\theta}\log\pi_\theta(a|s)\right)^2\right].\]
	In \cite{RN28}, the NPG update is
	\[
	w_\star \in \argmin_wL(w).
	\]
	We now show that the gradient of the loss function is the same as the one in the single agent setting. This is enough to show that the two algorithms coincide, since every other operation coincides. The only exception is the projection step, problem that can be side-stepped by projecting each single component of the gradient vector instead of the whole vector. For each agent $k$ we have that
	\begin{align*}
		\nabla_{\theta_k} L(w)&=\E_{s \sim d_{\nu}^{\pi},a \sim \pi_{\theta}(\cdot|s)}\left[\left(A^{\pi_{\theta}}(s,a)-w\cdot\nabla_{\theta}\log\pi_\theta(a|s)\right)\nabla_{\theta_k}\log\pi_\theta(a|s)\right]\\
		&=\sum_{j \in \K}\E_{s \sim d_{\nu}^{\pi},a \sim \pi_{\theta}(\cdot|s)}\left[\left(\frac{1}{K}A_j^{\pi_{\theta}}(s_j,a_j)-w_j\cdot\nabla_{\theta_j}\log\pi_\theta(a_j|s_j)\right)\nabla_{\theta_k}\log\pi_\theta(a|s)\right]\\
		&=E_{s_k \sim d_{\nu,k}^\pi,a_k \sim \pi^k_{\theta_k}(\cdot|s_k)} \left[\left(\frac{1}{K}A_k^{\pi_{\theta}}(s_k,a_k)-w_k\cdot\nabla_{\theta_k}\log\pi_\theta(a_k|s_k)\right)\nabla_{\theta_k}\log\pi_\theta(a_k|s_k)\right],
	\end{align*}
	which corresponds to the gradient for the single agent setting.
	
	With regards to guarantees, since the problem is decoupled, for any agent $k$ we have the same result as in Agarwal et al. 2020:
		\[\E\left[\min_{t<T}\{V_k^{\pi^\star}(\rho)-V_k^{(t)}(\rho)\}\right]\leq \frac{W}{1-\gamma}\sqrt{\frac{2\delta\log|\A_k|}{T}}+\sqrt{\frac{\kappa \varepsilon_{\text{stat},k}}{(1-\gamma)^3}}+\frac{\sqrt{\varepsilon_{\text{bias},k}}}{1-\gamma},\]
	where we assumed that
		\[\E\left[L_k(w^{(t)}_k,\theta^{(t)}_k,d^{(t)}_k)- L(w^{(t)}_{\star,k},\theta^{(t)}_k,d^{(t)}_k)|\theta^{(t)}_k\right] \leq\varepsilon_{\text{stat},k},\]
		\[\E\left[L(w^{(t)}_{\star,k},\theta^{(t)}_k,d^\star_k)\right] \leq\varepsilon_{\text{bias},k},\]
	where $d^{(t)}_k = d_{\nu}^{\pi_k^{(t)}} \pi^{{(t)}}_{k}$, $d^\star_k = d_{\nu}^{\pi_k^\star} \pi^{\star}_{k}$ and
	\[L_k(w,\theta_k,d)=\E_{s_k,a_k\sim d}\left[\left(A_k^{\pi_{\theta_k}}(s_k,a_k)-w\cdot\nabla_{\theta_k}\log\pi_{\theta_k}(a_k|s_k)\right)^2\right],\]
	\[w^{(t)}_{\star,k} \in \argmin_{\lVert w\rVert_2\leq W}L_k(w,\theta^{(t)}_k,d^{(t)}_k).\]

	The same result holds for the whole network:
	\begin{align*}
		\E\left[\min_{t<T}\{V^{\pi^\star}(\rho)-V^{(t)}(\rho)\}\right]&=\E\left[\min_{t<T}\left\{\frac{1}{K}\sum_{k \in \K}\left(V_k^{\pi^\star}(\rho)-V_k^{(t)}(\rho)\right)\right\}\right]\\
		&\leq \E\left[\min_{t<T}\max_{k\in\K}\left\{V_k^{\pi^\star}(\rho)-V_k^{(t)}(\rho)\right\}\right].
	\end{align*}
\end{proof}

\section{Proof of Theorem \ref{thm:main}}
\label{app:main}
We follow the proof in \cite{RN28} modifying it where necessary. We start by proving a modified NPG Regret Lemma.

\begin{lemma}
	Consider the setting of Theorem \ref{thm:main}, then we have that:
	\[
	\E\left[\min_{t<T}\{V^{\pi^\star}(\rho)-V^{(t)}(\rho)\}\right]\leq \frac{W}{1-\gamma}\sqrt{\frac{2\delta\max_{k\in\K}\log|\A_k|}{T}}+\E\left[\frac{1}{T(1-\gamma)}\sum_{t=0}^{T-1}\text{err}_t\right],
	\]
		where
	\[\text{err}_t=\E_{s \sim d^\star, a\sim\pi^\star(\cdot|s)}\left[A^{(t)}(s,a)-\frac{1}{K}\nabla_{\theta}\log\pi^{(t)}(a|s)\cdot w^{(t)}\right].\]
	
\end{lemma}
\begin{proof}
	We assume $\log\pi_{\theta}(a|s)$ is a $\delta$-smooth function of $\theta$. By smoothness we have:
	\begin{align*}
		\log\frac{\pi^{(t+1)}(a|s)}{\pi^{(t)}(a|s)}&\geq \nabla_{\theta}\log\pi^{(t)}(a|s)\cdot(\theta^{(t+1)}-\theta^{(t)})-\frac{\delta}{2}\norm{\theta^{(t+1)}-\theta^{(t)}}^2_2\\
		&=\eta\nabla_{\theta}\log\pi^{(t)}(a|s)\cdot w^{(t)} -\eta^2\frac{\delta}{2} \norm{w^{(t)}}^2_2.
	\end{align*}
	Then
	\begin{align*}
		&\frac{1}{K}\E_{s \sim d_\rho^{\pi^\star}}(\text{KL}(\pi^\star_s||\pi^{(t)}_s)- \text{KL}(\pi^\star_s||\pi^{(t+1)}_s))= \frac{1}{K}\E_{s \sim d^\star, a\sim\pi^\star(\cdot|s)}\left[\log\frac{\pi^{(t+1)}(a|s)}{\pi^{(t)}(a|s)}\right]\\
		&\geq \frac{\eta}{K}\E_{s \sim d^\star, a\sim\pi^\star(\cdot|s)}\left[\nabla_{\theta}\log\pi^{(t)}(a|s)\cdot w^{(t)} \right]-\eta^2\frac{\delta}{2K} \norm{w^{(t)}}^2_2
	\end{align*}
\begin{align*}
		&=\eta\E_{s \sim d^\star, a\sim\pi^\star(\cdot|s)}\left[A^{(t)}(s,a) \right]+\eta\E_{s \sim d^\star, a\sim\pi^\star(\cdot|s)}\left[\frac{1}{K}\nabla_{\theta}\log\pi^{(t)}(a|s)\cdot w^{(t)} -A^{(t)}(s,a)\right]\\&\quad-\eta^2\frac{\delta}{2K} \norm{w^{(t)}}^2_2\\
		&= (1-\gamma)\eta\left(V^{\pi^\star}(\rho)-V^{(t)}(\rho)\right)-\eta^2\frac{\delta}{2K}\norm{w^{(t)}}^2_2-\eta\text{ err}_t,
	\end{align*}
	where
	
	\[\text{err}_t=\E_{s \sim d^\star, a\sim\pi^\star(\cdot|s)}\left[A^{(t)}(s,a)-\frac{1}{K}\nabla_{\theta}\log\pi^{(t)}(a|s)\cdot w^{(t)}\right].\]
	
	Rearranging
	\begin{align*}
		V^{\pi^\star}(\rho)-V^{(t)}(\rho)\leq \frac{1}{1-\gamma}\left(\frac{1}{\eta K}\E_{s \sim d_\rho^{\pi^\star}}(\text{KL}(\pi^\star_s||\pi^{(t)}_s)- \text{KL}(\pi^\star_s||\pi^{(t+1)}_s))+\frac{\eta\delta}{2} W^2+\text{err}_t\right)
	\end{align*}
	and summing over t
	\begin{align*}
		V^{\pi^\star}(\rho)-\frac{1}{T}\sum_{t=0}^{T-1}V^{(t)}(\rho)\leq &\frac{1}{\eta KT(1-\gamma)}\sum_{t=0}^{T-1}\E_{s \sim d_\rho^{\pi^\star}}(\text{KL}(\pi^\star_s||\pi^{(t)}_s)- \text{KL}(\pi^\star_s||\pi^{(t+1)}_s))\\&+\frac{1}{T(1-\gamma)}\sum_{t=0}^{T-1}\left(\frac{\eta\delta}{2} W^2+\text{err}_t\right)\\
		\leq& \frac{\E_{s \sim d_\rho^{\pi^\star}}\text{KL}(\pi^\star_s||\pi^{(0)}_s)}{\eta KT(1-\gamma)}+\frac{\eta\delta W^2}{2(1-\gamma)}+\frac{1}{T(1-\gamma)}\sum_{t=0}^{T-1}\text{err}_t\\
		\leq& \frac{\log|\A|}{\eta KT(1-\gamma)}+\frac{\eta\delta W^2}{2(1-\gamma)}+\frac{1}{T(1-\gamma)}\sum_{t=0}^{T-1}\text{err}_t.
	\end{align*}
	Optimizing for $\eta$, we obtain the lemma.
\end{proof}

We can now prove Theorem \ref{thm:main}\\\\
\begin{proof}[of Theorem \ref{thm:main}]
	After using the NPG Regret Lemma we want to bound the $\text{err}_t$ term. We do so by dividing it in 3 parts. Let $(w^{(t)})_{t=1,\dots,T}$ be an update sequence such that, for each $t,k$, $(w_k^{(t)})_{N^r_{-k}}=0$, then 
	\begin{align*}
		\text{err}_t&=\E_{s \sim d^\star, a\sim\pi^\star(\cdot|s)}\left[A^{(t)}(s,a)-\frac{1}{K}\nabla_{\theta}\log\pi^{(t)}(a|s)\cdot w^{(t)}\right]\\
		&=\frac{1}{K}\sum_{k \in \K}\E_{s \sim d^\star, a\sim\pi^\star(\cdot|s)}\left[A^{(t)}_k(s,a)-\nabla_{\theta_k}\log\pi^{(t)}(a|s)\cdot w^{(t)}_k\right]\\
		&=\frac{1}{K}\sum_{k \in \K}\E_{s \sim d^\star, a\sim\pi^\star(\cdot|s)}\left[A^{(t)}_k(s,a)-\nabla_{(\theta_k)_{N^r_k}}\log\pi^{(t)}(a|s)\cdot (w_k^{(t)})_{N^r_{k}}\right]\\
		&=\frac{1}{K}\sum_{k \in \K}\E_{s \sim d^\star, a\sim\pi^\star(\cdot|s)}\left[\widetilde{A}^{(t)}_k(s_{N^{r}_k},a_{N^r_k})-\nabla_{(\theta_k)_{N^r_k}}\log\pi^{(t)}(a|s)\cdot (w_{k,\star}^{(t)})_{N^r_{k}}\right]\\
		&\quad+\frac{1}{K}\sum_{k \in \K}\E_{s \sim d^\star, a\sim\pi^\star(\cdot|s)}\left[A^{(t)}_k(s,a)-\widetilde{A}^{(t)}_k(s_{N^{r}_k},a_{N^r_k})\right]\\
		&\quad+\frac{1}{K}\sum_{k \in \K}\E_{s \sim d^\star, a\sim\pi^\star(\cdot|s)}\left[\nabla_{(\theta_k)_{N^r_k}}\log\pi^{(t)}(a|s)\cdot ((w^{(t)}_{k,\star})_{N^r_k}-(w_k^{(t)})_{N^r_{k}})\right],
	\end{align*}
	where $\forall k\in\K$
\[	w^{(t)}_{k,\star}\in \argmin_{\substack{\norm{w}_2\leq W\\ w(N^r_{-k})=0}}\E_{s \sim d^{(t)}, a\sim\pi^{(t)}(\cdot|s)}\left[\widetilde{A}^{(t)}_k(s_{N^{r}_k},a_{N^r_k})-\nabla_{(\theta_k)_{N^r_k}}\log\pi^{(t)}(a_k|s_{N^r_k})\cdot w\right]^2.\]
We now analyse each term separately.
	
	\paragraph{Term 1}
	\begin{align*}
		&\frac{1}{K}\sum_{k \in \K}\E_{s \sim d^\star, a\sim\pi^\star(\cdot|s)}\left[\widetilde{A}^{(t)}_k(s_{N^{r}_k},a_{N^r_k})-\nabla_{(\theta_k)_{N^r_k}}\log\pi^{(t)}(a|s)\cdot (w_{k,\star}^{(t)})_{N^r_{k}}\right]\\
		&\leq \frac{1}{K}\sum_{k \in \K}\sqrt{\E_{s \sim d^\star, a\sim\pi^\star(\cdot|s)}\left[\widetilde{A}^{(t)}_k(s_{N^{r}_k},a_{N^r_k})-\nabla_{(\theta_k)_{N^r_k}}\log\pi^{(t)}(a|s)\cdot (w_{k,\star}^{(t)})_{N^r_{k}}\right]^2}\leq \sqrt{\varepsilon_\text{bias}}
	\end{align*}

	\paragraph{Term 2}
	Firstly, we have that  $\forall k\in\K$
	\begin{align*}
		\widetilde{Q}^\pi_k(s_{N^{r}_k}, a_{N^r_k}) &=\E\left[\sum_{t=0}^\infty \gamma^t r_k(s_k(t), a_k(t))\big|\pi_{N^r_k}, s_{N^{r}_k}(0)=s_{N^{r}_k}, a_{N^r_k}(0)=a_{N^r_k}\right]
	\end{align*}
\begin{align*}
		&=\E\left[\E\left[\sum_{t=0}^\infty \gamma^t r_k(s_k(t), a_k(t))\big|\pi_{N^r_k}, s_{N^{r}_k}(0)=s_{N^{r}_k}, a_{N^r_k}(0)=a_{N^r_k},\right.\right.\\&\qquad\qquad\qquad\qquad\qquad\qquad\qquad\left.\left.s_{N^{r}_{-k}}(0), a_{N^{r}_{-k}}(0)\right]\right]\\
		&=\E\left[Q_k^\pi\left(s_{N^{r}_k}, a_{N^r_k}, s_{N^{r}_{-k}}(0), a_{N^{r}_{-k}}(0)\right)\right],
	\end{align*}
	for some distribution of $s(0)_{N^{r}_{-k}}$ and $ a(0)_{N^{r}_{-k}}$. Similarly  $\forall k\in\K$
	\[V^\pi_k(s)-\widetilde{V}^\pi_k(s_{N^r_k})= \E\left[V^\pi_k(s)-V_k^\pi\left(s_{N^{r}_k}, s(0)_{N^{r}_{-k}}\right)\right].\]
	So
	\[A^\pi_k(s,a)-\widetilde{A}^\pi_k(s_{N^{r}_k},a_{N^r_k})= 
	Q^\pi_k(s,a)-\widetilde{Q}^\pi_k(s_{N^{r}_k},a_{N^r_k})-
	V^\pi_k(s)+\widetilde{V}^\pi_k(s_{N^{r}_k}),\]
	\begin{align*}
		V^\pi_k(s)-\widetilde{V}^\pi_k(s_{N^r_k})&= \E\left[V^\pi_k(s)-V_k^\pi\left(s_{N^{r}_k}, s(0)_{N^{r}_{-k}}\right)\right]\\
		&\leq \E\left[c'\xi^{r+1}\right]=c'\xi^{r+1}
	\end{align*}
	and
	\begin{align*}
		Q^\pi_k(s,a)-\widetilde{Q}^\pi_k(s_{N^r_k},a_{N^r_k})&= \E\left[Q^\pi_k(s,a)-Q_k^\pi\left(s_{N^{r}_k}, a_{N^r_k}, s(0)_{N^{r}_{-k}}, a(0)_{N^{r}_{-k}}\right)\right]\\
		&\leq \E\left[c\rho^{r+1}\right]=c\rho^{r+1}.
	\end{align*}
	Then  $\forall k\in\K$
	\[\left|A^\pi_k(s,a)-\widetilde{A}^\pi_k(s_{N^r_k},a_{N^r_k})\right|\leq c\rho^{r+1}+c'\xi^{r+1}.\]
	\paragraph{Term 3}
	In this paragraph, we denote, $\forall k\in\K$, $ w^{(t)}_{k,\star}= (w^{(t)}_{k,\star})_{N^r_k}$ and $w^{(t)}_k = (w_k^{(t)})_{N^r_{k}}$ for clarity of exposition. Remember that \[\Sigma_{\nu,k}^\theta=\E_{s,a \sim \nu}\left[\nabla_{(\theta_k)_{N_{k}^r}}\log\pi_\theta(a_k|s_{N_k^r})(\nabla_{(\theta_k)_{N_{k}^r}}\log\pi_\theta(a_k|s_{N_k^r}))^\top\right]\] and that we assume, $\forall k$, \[\sup_{w\in\R^d}\frac{w^\top\Sigma^{(t)}_{d^\star,k}w}{w^\top\Sigma^{(t)}_{\nu,k}w} \leq \kappa'.\]
	Then  $\forall k\in\K$
	\begin{align*}
		&\E_{s \sim d^\star, a\sim\pi^\star(\cdot|s)}\left[\nabla_{(\theta_k)_{N^r_k}}\log\pi^{(t)}(a|s)\cdot (w^{(t)}_{k,\star}-w^{(t)}_k)\right]\\&\leq\sqrt{(w^{(t)}_{k,\star}-w^{(t)}_k)^\top\Sigma^{(t)}_{d^\star,k} (w^{(t)}_{k,\star}-w^{(t)}_k)}\\
		&=\sqrt{\frac{(w^{(t)}_{k,\star}-w^{(t)}_k)^\top\Sigma^{(t)}_{d^\star,k} (w^{(t)}_{k,\star}-w^{(t)}_k)}{(w^{(t)}_{k,\star}-w^{(t)}_k)^\top\Sigma^{(t)}_{\nu,k} (w^{(t)}_{k,\star}-w^{(t)}_k)}(w^{(t)}_{k,\star}-w^{(t)}_k)^\top\Sigma^{(t)}_{\nu,k} (w^{(t)}_{k,\star}-w^{(t)}_k)}\\
		&\leq\sqrt{\kappa'(w^{(t)}_{k,\star}-w^{(t)}_k)^\top\Sigma^{(t)}_{\nu,k} (w^{(t)}_{k,\star}-w^{(t)}_k)}\\
		&[\text{using that $(1-\gamma)\nu\leq d^{\pi^{(t)}}_\nu$}]\\
		&\leq \sqrt{\frac{\kappa'}{1-\gamma}(w^{(t)}_{k,\star}-w^{(t)}_k)^\top\Sigma_{d^{(t)}}^{(t)} (w^{(t)}_{k,\star}-w^{(t)}_k)}.
	\end{align*}
	Since $w^{(t)}_{k,\star}$ optimizes $\widetilde{L}_k(w, \theta^{(t)}, d^{(t)})$ the first order optimality condition implies that, $\forall w$,  $\forall k\in\K$,
	\[(w-w^{(t)}_{k,\star})\cdot \nabla \widetilde{L}_k(w^{(t)}_{k,\star}, \theta^{(t)}, d^{(t)})\geq 0.\]	
	So $\forall w$,  $\forall k\in\K$,
	\begin{align*}
		&\widetilde{L}_k(w, \theta^{(t)}, d^{(t)})-\widetilde{L}_k(w^{(t)}_{k,\star}, \theta^{(t)}, d^{(t)})\\
		&= \E_{s \sim d^{(t)}, a\sim\pi^{(t)}(\cdot|s)}\left[\widetilde{A}_k^{(t)}(s_{N^{r}_k},a_{N^r_k})-\nabla_{(\theta_k)_{N^r_k}}\log\pi^{(t)}(a|s)\cdot w\right.\\&\left.\qquad\qquad\qquad\qquad\quad+\nabla_{(\theta_k)_{N^r_k}}\log\pi^{(t)}(a|s)\cdot w^{(t)}_{k,\star}-\nabla_{(\theta_k)_{N^r_k}}\log\pi^{(t)}(a|s)\cdot w^{(t)}_{k,\star}\right]^2\\&\qquad-\widetilde{L}_k(w^{(t)}_{k,\star}, \theta^{(t)}, d^{(t)})\\
		&=\E_{s \sim d^{(t)}, a\sim\pi^{(t)}(\cdot|s)}\left[\nabla_{(\theta_k)_{N^r_k}}\log\pi^{(t)}(a|s)\cdot w^{(t)}_{k,\star}-\nabla_{(\theta_k)_{N^r_k}}\log\pi^{(t)}(a|s)\cdot w\right]^2\\	
		&+2(w-w^{(t)}_{k,\star})\E_{s \sim d^{(t)}, a\sim\pi^{(t)}(\cdot|s)}\left[\left(\widetilde{A}_k^{(t)}(s_{N^{r}_k},a_{N^r_k})-\nabla_{(\theta_k)_{N^r_k}}\log\pi^{(t)}(a|s)\cdot w^{(t)}_{k,\star}\right)\right.\\&\qquad\qquad\qquad\qquad\qquad\qquad\qquad\left.\cdot\nabla_{(\theta_k)_{N^r_k}}\log\pi^{(t)}(a|s)\right]\\
		&=(w^{(t)}_{k,\star}-w^{(t)}_k)^\top\Sigma_{d^{(t)}}^{(t)} (w^{(t)}_{k,\star}-w^{(t)}_k)+(w-w^{(t)}_{k,\star})\cdot \nabla \widetilde{L}_k(w^{(t)}_{k,\star}, \theta^{(t)}, d^{(t)})
	\end{align*}
\begin{align*}
		&\geq (w^{(t)}_{k,\star}-w^{(t)}_k)^\top\Sigma_{d^{(t)}}^{(t)} (w^{(t)}_{k,\star}-w^{(t)}_k).
	\end{align*}
	Therefore
	\begin{align*}
		&\E\left[\E_{s \sim d^\star, a\sim\pi^\star(\cdot|s)}\left[\nabla_{(\theta_k)_{N^r_k}}\log\pi^{(t)}(a|s)\cdot (w^{(t)}_{k,\star}-w^{(t)}_k)\right]\right]\\
		&\leq\sqrt{\E\left[\frac{\kappa'}{1-\gamma}(\widetilde{L}_k(w, \theta^{(t)}, d^{(t)})-\widetilde{L}_k(w^{(t)}_{k,\star}, \theta^{(t)}, d^{(t)}))\right]}\\
		&=\sqrt{\E\left[\frac{\kappa'}{1-\gamma}\E\left[(\widetilde{L}_k(w, \theta^{(t)}, d^{(t)})-\widetilde{L}_k(w^{(t)}_{k,\star}, \theta^{(t)}, d^{(t)}))\big| \theta^{(t)}\right]\right]}\\
		&\leq\sqrt{\frac{\kappa'\varepsilon_{\text{stat}}}{1-\gamma}},
	\end{align*}
	which completes the proof.
\end{proof}

\section{Statistical Error}
\label{app:stat}
Assume access to a global clock, then Algorithm \ref{alg:sampler} is a sampler of $d^{(t)}$ and an unbiased sampler of $\widetilde{A}_k^{(t)}(s_{N^{r}_k},a_{N^r_k})$, for every $k\in\K$.

\begin{proposition}
	Consider the setting of Theorem \ref{thm:main} and assume access to the sampler in Algorithm \ref{alg:sampler}. Assume that $\norm{\nabla_{(\theta_k)_{N_{k}^r}}\log\pi^{(t)}(a|s)}_2\leq B$. Then, solving the minimization problem in \eqref{eq:app_update} with stochastic projected gradient descent for $N$ steps and step size $\alpha = W/(8B(BW+\frac{1}{1-\gamma})\sqrt{N})$ gives
	\[\varepsilon_{\text{stat}}\leq \frac{8BW(BW+\frac{1}{1-\gamma})}{\sqrt{N}}.\]
\end{proposition}
\begin{proof}
	The proposition follows from a result on stochastic projected gradient descent \citep{RN69}. Consider the minimization problem $\min_{x\in C}f(x)$ for a convex function $f$, then performing the update
	\[x_{t+1}=P_C(x_t-\alpha v_t),
	\]
	where $C=\{x:\norm{x}_2\leq W\}$ for some $W\geq 0$, $P_C(\cdot)$ is the projection on $C$, $v_t$ is such that $\E[v_t|x_t]=\nabla f(x_t)$ and $\norm{v_t}\leq\rho$, for $N$ steps and with $\alpha=\sqrt{\frac{W^2}{\rho^2 N}}$, gives
	\[\E\left[f\left(\frac{1}{N}\sum_{t=1}^{N}x_t\right)\right]-f(x^\star)\leq \frac{W\rho}{\sqrt{N}}.\]
	Noticing that $\widetilde{A}^{(t)}_k(s,a)\leq2/(1-\gamma)$ and that the sampled gradient is therefore bounded by $8B(BW+\frac{1}{1-\gamma})$ gives the proposition.
\end{proof}

The minimization problem in \eqref{eq:app_update} can therefore be solved by each agent $k$ through stochastic projected gradient descent, which only depends on the states and actions of $N^{r}_k$ and on the parameters $(\theta_k)_{N_{k}^r}$ and with a computational cost that does not depend on $K$.

\begin{algorithm}
	\caption{Sampler}
	\label{alg:sampler}
	\KwIn{Starting state-action distribution $\nu$}
	For each agent $k\in\mathcal{K}$ set $\widehat{Q}_k = 0$, $\widehat{V}_k=0$.\\
	$\forall k \in \mathcal{K}$, sample $s_k(0),a_k(0)\sim \nu_k$ and start at state $s_k(0)$\;
	{\it ($d^\pi_\nu$ sampling)} \textbf{At every time-step $h\geq0$}:\\
	\Indp with probability $\gamma$, $\forall k \in \K$ execute $a_k(h)$, transition to $s_k(h+1)$ and sample  $a_k(h+1)\sim\pi_k(\cdot|s_{N^r_k}(h+1))$\;
	else accept $(s(h),a(h))$ as the sample and exit the loop, each agent only saves $\left(s_{N^r_k}(h),a_{N^r_k}(h)\right)$.\\
	\Indm
	Set SampleQ=True with probability 1/2\;
	\If{SampleQ=True}{
		$\forall k \in \K$ execute $a_k(h)$ and then, for every time-step $h'>h$ with termination probability $\gamma$, transition to $s_k(h')$, sample $a_k(h'+1)\sim\pi_k(\cdot|s_{N^r_k}(h'+1))$ and set $\widehat{Q}_k = \widehat{Q}_k + r(s_k(h'), a_k(h'))$\;
	}
	\Else{
		$\forall k \in \K$ sample $a_k(h)\sim\pi_k(\cdot|s_{N^r_k}(h))$ and execute $a_k(h)$ and then, for every time-step $h'>h$ with termination probability $\gamma$, transition to $s_k(h')$, sample $a_k(h'+1)\sim\pi_k(\cdot|s_{N^r_k}(h'+1))$ and set $\widehat{V}_k = \widehat{V}_k + r(s_k(h'), a_k(h'))$\;
	}
	\KwResult{$(s(h),a(h))$ and $\left(\widehat{A}_k\left(s_{N^r_k}(h),a_{N^r_k}(h)\right)=2\left(\widehat{Q}_k-\widehat{V}_k\right)\right)_{k \in \K}$.}
\end{algorithm}

\section{Bias analysis}
\label{app:bias}

Let $\widetilde{L}_k^r$ and $w_{k,\star}$ be defined as in Section \ref{main}. For every $k\in\K$, let
$$
{L}_k(w, \theta, \nu)= \E_{s,a \sim \nu}\left[\left({A}^{\pi_{\theta}}_k(s,a)-\nabla_{\theta_k}\log\pi_{\theta_k}(a_k|s)\cdot w\right)^2\right].
$$
Let $\varepsilon_{\text{bias},r}$ be the smallest possible localized bias term associated to the range parameter $r$, i.e.
$$
\varepsilon_{\text{bias},r}=\max_{k\in\mathcal{K}}\widetilde{L}_k^r(w_{k,\star}, \theta^{(t)},d^\star),
$$
which is the same assumption as in Theorem 9, but with an equality instead of an inequality. Let $\varepsilon_\text{bias}$ be the smallest possible non-localized bias term associated with an infinite range parameter, i.e. the case where we make no approximation, namely,
$$
\varepsilon_{\text{bias}}=\max_{k\in\mathcal{K}}{L}_k(w'_{k,\star}, \theta^{(t)},d^\star)
$$
where
$$
w'_{k,\star}\in\argmin_{\lVert w\rVert_2\leq W}{L}_k(w, \theta^{(t)},d^{(t)}).
$$

\begin{proposition}
	\label{prop:bias}
	Assume that
	$$
	\left\lVert\nabla_{\theta_k}\log\pi^{(t)}(a|s)\right\rVert_2\leq B
	$$
	and that
	$$
	\left\lVert\nabla_{(\theta_k)_{N^r_{-k}}} \pi_\theta(a_k|s)\right\rVert_2 \leq \omega_r.
	$$
	Then
	$$
	\varepsilon_{\text{bias},r}\leq\varepsilon_{\text{bias}}+e_r+2\left(\frac{2}{1-\gamma}+WB\right)\sqrt{\frac{2\kappa' e_r}{1-\gamma}}+\frac{2\kappa' e_r}{1-\gamma},
	$$
	where
	$$
	e_r=\left(\frac{4}{1-\gamma}+2WB\right)W\omega_r+\left(\frac{4}{1-\gamma}+2WB\right)\left(c\psi^{r+1}+c'\phi^{r+1}\right).
	$$
\end{proposition}

The second assumption can be respected by choosing a policy belonging to the policy class described in Appendix \ref{app:pol_ex}, as $\omega_r$ can be controlled by setting the parameters $\{\alpha_{r}\}$ small enough. In particular, it is possible to set it to be exponentially small in $r$. Therefore, the proposition shows that the localized bias is at most the bias associated with an infinite range parameter plus a quantity that goes to 0 exponentially fast in $r$. We present the proof of this result below.

\begin{proof}[of Proposition \ref{prop:bias}]
	To prove the proposition we need two intermediate results. For any $w$,$\theta$ and $\nu$, we have that
	\begin{align*}
		&\left|\widetilde{L}^r_k((w)_{N^r_{k}},\theta, \nu)-L_k(w, \theta, \nu)\right|\\
		&=\E_{s,a \sim \nu}\left[\left(\widetilde{A}^{\pi_\theta}_k(s_{N^r_k},a_{N^r_k})\right)^2-\left({A}^{\pi_{\theta}}_k(s,a)\right)^2 \right]\\
		&\quad+\E_{s,a \sim \nu}\left[\left(\nabla_{(\theta_k)_{N^r_{k}}}\log\pi^{(t)}(a|s)\cdot (w)_{N^r_{k}}\right)^2 - \left(\nabla_{\theta_k}\log\pi_{\theta_k}(a_k|s)\cdot w\right)^2\right]
	\end{align*}
	\begin{align*}
		&\quad+2\E_{s,a \sim \nu}\left[{A}^{\pi_{\theta}}_k(s,a)\nabla_{\theta_k}\log\pi_{\theta_k}(a_k|s)\cdot w-\widetilde{A}^{\pi_\theta}_k(s_{N^r_k},a_{N^r_k})\nabla_{(\theta_k)_{N^r_{k}}}\log\pi^{(t)}(a|s)\cdot (w)_{N^r_{k}}\right]\\
		&\leq \frac{4}{1-\gamma}\left(c\psi^{r+1}+c'\phi^{r+1}\right) + 2W^2B\omega_r + 2WB\left(c\psi^{r+1}+c'\phi^{r+1}\right) + \frac{4}{1-\gamma}W\omega_r\\
		&= e_r.
	\end{align*}
	Following the same passages in the proof of Theorem \ref{thm:main} in Appendix \ref{app:main}, we have that
	\begin{align*}
		&\max_{k\in\mathcal{K}}\E_{s,a \sim d^\star}\left[\nabla_{(\theta_k)_{N^r_{k}}}\log\pi^{(t)}(a|s_{N^r_{k}})^\top\left((w'_{k,\star})_{N^r_{k}}-w_{k,\star}\right)\right]\\
		&\leq\sqrt{\frac{\kappa'}{1-\gamma}(\widetilde{L}^r_k((w'_{k,\star})_{N^r_{k}}, \theta^{(t)}, d^{(t)})-\widetilde{L}^r_k(w_{k,\star}, \theta^{(t)}, d^{(t)}))}\\
		&[\text{using the first result}]\\
		&\leq\sqrt{\frac{\kappa'}{1-\gamma}\left|{L}_k(w'_{k,\star}, \theta^{(t)}, d^{(t)})-\widetilde{L}^r_k(w_{k,\star}, \theta^{(t)}, d^{(t)})\right|+\frac{\kappa' e_r}{1-\gamma}}\\
		&	=\sqrt{\frac{\kappa'}{1-\gamma}\left|\min_{\lVert w\rVert_2\leq W}{L}_k(w, \theta^{(t)}, d^{(t)})-\min_{\lVert w\rVert_2\leq W}\widetilde{L}^r_k((w)_{N^r_{k}}, \theta^{(t)}, d^{(t)})\right|+\frac{\kappa' e_r}{1-\gamma}}\\
		&\leq\sqrt{\frac{2\kappa' e_r}{1-\gamma}}.
	\end{align*}
	We can now prove the proposition.
	\begin{align*}
		\varepsilon_{\text{bias},r}&=\max_{k\in\mathcal{K}}\E_{s,a \sim d^\star}\left[\left(\widetilde{A}^{(t)}_k(s_{N^r_k},a_{N^r_k})- \nabla_{(\theta_k)_{N^r_{k}}}\log\pi^{(t)}(a|s_{N^r_{k}})\cdot w_{k,\star}\right)^2\right]\\
		&=\max_{k\in\mathcal{K}}\E_{s,a \sim d^\star}\left[\left(\widetilde{A}^{(t)}_k(s_{N^r_k},a_{N^r_k})- \nabla_{(\theta_k)_{N^r_{k}}}\log\pi^{(t)}(a|s_{N^r_{k}})\cdot (w'_{k,\star})_{N^r_{k}}\right)^2\right.\\
		&\qquad\qquad\qquad\quad+2\left(\widetilde{A}^{(t)}_k(s_{N^r_k},a_{N^r_k})-\nabla_{(\theta_k)_{N^r_{k}}}\log\pi^{(t)}(a|s_{N^r_{k}})\cdot (w'_{k,\star})_{N^r_{k}}\right)\\
		&\qquad\qquad\qquad\qquad\quad\cdot\nabla_{(\theta_k)_{N^r_{k}}}\log\pi^{(t)}(a|s_{N^r_{k}})^\top\left((w'_{k,\star})_{N^r_{k}}-w_{k,\star}\right)\\
		&\left.\qquad\qquad\qquad\quad+\left(\nabla_{(\theta_k)_{N^r_{k}}}\log\pi^{(t)}(a|s_{N^r_{k}})^\top\left((w'_{k,\star})_{N^r_{k}}-w_{k,\star}\right)\right)^2\right]
	\end{align*}
\begin{align*}
		&\leq \max_{k\in\mathcal{K}}\E_{s,a \sim d^\star}\left[\left(\widetilde{A}^{(t)}_k(s_{N^r_k},a_{N^r_k})- \nabla_{(\theta_k)_{N^r_{k}}}\log\pi^{(t)}(a|s_{N^r_{k}})\cdot (w'_{k,\star})_{N^r_{k}}\right)^2\right]\\
		&\quad+2\left(\frac{2}{1-\gamma}+WB\right)\sqrt{\frac{2\kappa' e_r}{1-\gamma}}+\frac{2\kappa' e_r}{1-\gamma}\\
		&= \max_{k\in\mathcal{K}}\widetilde{L}_k^r((w'_{k,\star})_{N^r_{k}}, \theta^{(t)},d^\star)+2\left(\frac{2}{1-\gamma}+WB\right)\sqrt{\frac{2\kappa' e_r}{1-\gamma}}+\frac{2\kappa' e_r}{1-\gamma}\\
		&\leq\varepsilon_{\text{bias}}+e_r+2\left(\frac{2}{1-\gamma}+WB\right)\sqrt{\frac{2\kappa' e_r}{1-\gamma}}+\frac{2\kappa' e_r}{1-\gamma}.
	\end{align*}
\end{proof}

\end{document}